\documentclass{article}

\PassOptionsToPackage{numbers, compress}{natbib}
 \usepackage[main, final]{neurips_2026}

\usepackage{xspace}
\newcommand{\ours}{\textsc{PoEM}\xspace}
\newcommand{\methodm}{\mathrm{PoEM}}
\usepackage{algorithm}
\usepackage{algpseudocode}

\usepackage[utf8]{inputenc} % allow utf-8 input
\usepackage[T1]{fontenc}    % use 8-bit T1 fonts
\usepackage{hyperref}       % hyperlinks
\usepackage{url}            % simple URL typesetting
\usepackage{booktabs}       % professional-quality tables
\usepackage{amsmath}        % math
\usepackage{amsfonts}       % blackboard math symbols
\usepackage{amssymb}        % \gtrsim, \lesssim, etc.
\usepackage{bm}             % bold math symbols
\usepackage{xspace}         % spacing for \eg, \ie macros
\usepackage{nicefrac}       % compact symbols for 1/2, etc.
\usepackage{microtype}      % microtypography
\usepackage{xcolor}         % colors
\usepackage{todonotes}      % for \todo
\usepackage{graphicx}       % for \includegraphics
\usepackage{enumitem}

\usepackage{amsthm}
\theoremstyle{plain}
\newtheorem{theorem}{Theorem}[section]
\newtheorem{proposition}[theorem]{Proposition}

\theoremstyle{definition}

\theoremstyle{remark}

\newcommand{\myparagraph}[1]{\vspace{0pt} \noindent \textbf{#1}}

\newcommand{\draftcomment}[1]{\todo{#1}}
\newcommand{\showkh}{}  % comment out this line to hide all \kh{} comments
\makeatletter
\ifdefined\showkh
  \DeclareRobustCommand{\kh}[1]{\ifx\@captype\@undefined\draftcomment{\textcolor{blue}{\small [KH: #1]}}\else\textcolor{blue}{\small [KH: #1]}\fi}
\else
  \newcommand{\kh}[1]{\unskip}
\fi
\makeatother

\definecolor{cPGRPO}{HTML}{e99638}
\definecolor{cPDPO}{HTML}{184F95}
\definecolor{cRMRB}{HTML}{C0392B}
\definecolor{cRMHH}{HTML}{4A3AA7}
\definecolor{cRMPUB}{HTML}{0F7A55}
\definecolor{cSD}{HTML}{52514E}
\definecolor{cRMSKY}{HTML}{1E9DC8}
\newcommand{\basisname}[2]{\textcolor{#1}{\texttt{#2}}\xspace}
\newcommand{\pgrpo}{\basisname{cPGRPO}{P-GRPO}}
\newcommand{\pdpo}{\basisname{cPDPO}{P-DPO}}
\newcommand{\rmrb}{\basisname{cRMRB}{RM-RB}}
\newcommand{\rmhh}{\basisname{cRMHH}{RM-HH}}
\newcommand{\rmpub}{\basisname{cRMPUB}{RM-Div}}
\newcommand{\sdddpo}{\basisname{cSD}{SD-DDPO}}
\newcommand{\rmsky}{\basisname{cRMSKY}{RM-Sky}}

\makeatletter
\newcommand\FloatBarrier{\par\begingroup \let\@elt\relax
  \edef\@tempa{\@botlist\@deferlist\@dbldeferlist}%
  \ifx\@tempa\@empty \else
    \ifx\@fltovf\relax \clearpage
    \else \newpage \let\@fltovf\relax \FloatBarrier \fi
  \fi \endgroup}
\makeatother
\title{
PoEM: \\
Predicting RL Outcomes from Existing Policies \\
}

\author{
  Kimia Hamidieh\thanks{Correspondence to \texttt{hamidieh@mit.edu}} \\
  MIT CSAIL
  \And
  Giannis Daras \\
  MIT CSAIL
  \And
  Antonio Torralba \\
  MIT CSAIL
}

\begin{document}

\maketitle

\begin{abstract}
  Foundation models are post-trained with reinforcement learning (RL) to maximize specific rewards, such as helpfulness, correctness, or instruction following. This post-training process is computationally intensive, sometimes unstable, and has to be run from scratch every time the reward model changes or when we want to combine multiple rewards. We hence ask: given a new reward function, \textit{is it possible to predict the policy RL would produce without actually training on that reward}? We answer this in the affirmative by introducing \ours, a framework for predicting RL outcomes using models previously trained on other rewards. First, we build on the fact that if the new reward function can be written as a linear combination of existing ones, then the optimal policy can be expressed as a combination of the existing policies in log-probability space. Surprisingly, even in cases where the rewards are \textit{not} linearly connected, we observe that policies learned through RL often span an approximately low-rank subspace across rewards. To our benefit, the weights for this combination can be estimated using only the reward or basis policy outputs on the samples. 
We turn these observations into an algorithm that takes post-trained models and a target reward function, and approximates the target RL policy without any additional RL training. 
We experimentally evaluate our approach across programmatic and learned rewards for open-ended generation, with additional experiments on math reasoning, tool calling, and image generation.

\end{abstract}

\section{Introduction}
Model post-training has become a crucial step in adapting frontier models to align with human preferences and achieve desirable outcomes~\citep{christiano2017deep,ouyang2022traininglanguagemodelsfollow}. Unfortunately, post-training through Reinforcement Learning (RL) is expensive, sometimes unstable, and has to be redone every time the reward changes.
Current heuristics, such as averaging adapter weights~\citep{rame2023rewarded,jang2023personalized,ilharco2022editing}, are used with combination weights chosen by the user and degrade as the number of models to be combined increases~\citep{yadav2023ties,yang2023adamerging}, and alternative methods like best-of-$N$ sampling~\citep{nakano2021webgpt,cobbe2021training} remain effective only within narrow regimes. Our broad motivation is to predict and preview the RL post-trained model on a new reward, without training. Concretely, we are interested in the following question:
\vspace{-0.5em}
\begin{center}
\parbox{0.9\linewidth}{\centering\emph{Given a basis of models trained with RL on different rewards, can we predict the RL outcome on a new target reward without training?}}
\end{center}
\vspace{-0.5em}

We start by making the simple theoretical observation that if the new reward is a linear combination of the existing rewards, there is a convenient closed-form solution for the optimum policy, which is a weighted product of the reference model and the policies that have been trained on the existing rewards. We further show that the weights of this product can be estimated through linear regression on a small set of reward evaluations. 
This observation allows us to approximate the RL outcomes on linear combinations of existing rewards without actually running RL.

\begin{figure}[t]
    \centering
    \includegraphics[width=\linewidth]{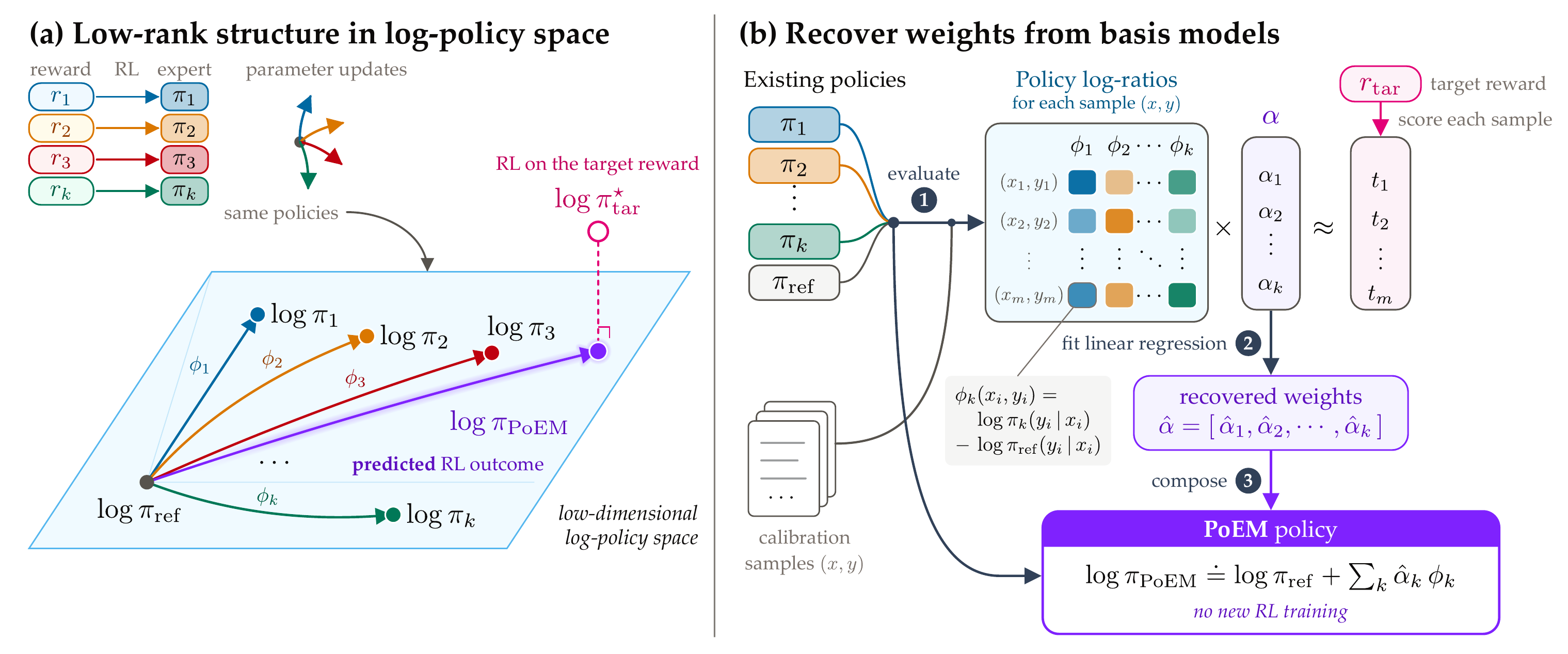}
    \caption{\textbf{Overview of \ours.} (a)~Policies $\pi_1,\dots,\pi_k$ post-trained with RL from a shared reference policy $\pi_{\mathrm{ref}}$ on different rewards have nearly orthogonal parameter updates, yet their policy log-ratios $\phi_k=\log\pi_k-\log\pi_{\mathrm{ref}}$ span a low-dimensional space.
    The RL outcome $\pi^*_{\mathrm{tar}}$ on a target reward can lie close to this space even when that reward is not a combination of the basis rewards, so a composition $\pi_{\methodm}$ of the basis approximates it. (b)~\ours evaluates the basis log-ratios on shared calibration samples $(x,y)$, centers them within each prompt, and fits weights $\hat{\bm{\alpha}}$ by least squares to predict the centered scores of the target reward $r_{\mathrm{tar}}$ on the same samples. The \ours policy $\pi_{\methodm}$ composes the basis with these weights at inference time, with no new RL training. 
    }
    \vspace{-0.8em}
    \label{fig:overview}
\end{figure}

What if we want to simulate a completely new reward? The new reward rarely decomposes linearly into a small fixed basis. Despite this, we find that the policies learned by existing rewards often span a lower-dimensional space in log-probabilities than the reward geometry suggests. Even when the weight updates are nearly orthogonal, the matrix of log-ratios against the reference model~\citep{rafailov2023direct} spans far fewer effective directions, often fewer even than the rewards that produced them.
A policy trained on a new reward can therefore lie inside this subspace even when the reward is not a linear combination of the basis rewards. The strong linearity assumption on rewards is therefore replaced by a much weaker geometric assumption on policy log-ratios.

Inspired by these observations, we propose \ours (Product of Experts Mixing), a framework for predicting RL outcomes for new rewards from models previously trained with RL on a different set of rewards (Fig.~\ref{fig:overview}). We estimate composition weights from a small set of samples scored under the new reward, using either the basis reward outputs or the basis policy log-ratios. At inference time, the policy we output is a weighted log mixture of the base model and the previously obtained policies~\citep{liu2021dexperts,liu2024tuning,liu2024decoding}. 

We evaluate \ours on models trained with GRPO, DPO and PPO across programmatic rewards and on reward models, including math reasoning and tool calling. When the new reward combines the basis rewards, \ours recovers most of RL's reward gain, often with errors comparable to a second RL run. For new rewards, a coverage score computed before decoding specifies which ones \ours can reach.
Our contributions are as follows:
\begin{itemize}[topsep=0pt, parsep=0pt, itemsep=0pt, partopsep=0pt]
    \item We propose \ours, a method that approximates the RL outcome on a new reward by composing existing expert policies trained on different rewards, with no further training.
    \item We propose a regression recipe for recovering composition weights, using either basis reward scores or basis log-ratios, and a coverage score that ranks in advance which new rewards the basis can cover.
    \item We find a policy space rank gap: near-orthogonal experts span far fewer directions in policy space than in parameter space, and fewer than the rewards that trained them.
    \item We evaluate PoEM on combined and held-out rewards and compare its predictions against directly trained RL policies and existing baselines.
\end{itemize}

\section{Predicting RL Outcomes by Policy Composition}
\label{sec:method}

In this section, we propose \ours to predict the policy that RL training on a new target reward would produce, as formalized in Section~\ref{sec:method-setup}. We first consider the case in which the target reward is a linear combination of the rewards the basis policies were trained on (Section~\ref{sec:method-poe}), and instantiate the resulting composition for autoregressive and diffusion models (Section~\ref{sec:method-instantiations}). We then explain how to recover the composition weights from the basis rewards or from the basis policies alone (Section~\ref{sec:method-regress}), and show how to check whether the basis policies cover the new reward (Section~\ref{sec:method-coverage}).

\subsection{Problem setup}
\label{sec:method-setup}
Let $\pi_{\mathrm{ref}}$ be a reference model and let $\mathcal{B}=\{\pi_1,\ldots,\pi_n\}$ be a basis of policies obtained by post-training the same reference model on different rewards. Basis policy $\pi_k$ is trained to increase reward $r_k$ while staying close to the reference policy using the following objective~\citep{ziegler2019fine,ouyang2022traininglanguagemodelsfollow}
\begin{equation}
J_{r_k}(\pi) = \mathbb{E}_{x\sim\mathcal{D}} \left[ \mathbb{E}_{y\sim\pi(\cdot\mid x)}[r_k(x,y)] - \beta\,\mathrm{KL}\!\left( \pi(\cdot\mid x)\,\|\,\pi_{\mathrm{ref}}(\cdot\mid x) \right) \right],
\label{eq:rlhf-obj}
\end{equation}
with a shared reference policy and KL coefficient $\beta>0$. Given a new \emph{target} reward $r_{\mathrm{tar}}$, our goal is to approximate the policy $\pi^*_{\mathrm{tar}}\in\arg\max_\pi J_{r_{\mathrm{tar}}}(\pi)$ without training.

\subsection{Exact composition for linear rewards}
\label{sec:method-poe}

We first study the idealized case in which the target reward is a linear combination of the basis rewards,
\begin{equation}
r_{\mathrm{tar}}(x,y)=\sum_{k=1}^{n}\alpha_k r_k(x,y).
\label{eq:linear-reward}
\end{equation}
For an unrestricted policy class, the optimizer of Eq.~\eqref{eq:rlhf-obj} for any reward $r$ is the exponentially tilted reference policy \citep{peng2019advantage,korbak2022rl,rafailov2023direct},
\begin{equation}
\pi_r^*(y\mid x) = \frac{1}{Z_r(x)}\, \pi_{\mathrm{ref}}(y\mid x) \exp\!\left(\frac{r(x,y)}{\beta}\right), \qquad Z_r(x) = \mathbb{E}_{y\sim\pi_{\mathrm{ref}}(\cdot\mid x)} \left[\exp\!\left(\frac{r(x,y)}{\beta}\right)\right].
\label{eq:boltz}
\end{equation}
If every basis policy is the exact optimizer for its reward, $\pi_k=\pi^*_{r_k}$, substituting Eq.~\eqref{eq:linear-reward} into Eq.~\eqref{eq:boltz} gives:
\begin{equation}
\pi^*_{\mathrm{tar}}(y\mid x) \propto \pi_{\mathrm{ref}}(y\mid x) \prod_{k=1}^{n} \left( \frac{\pi_k(y\mid x)}{\pi_{\mathrm{ref}}(y\mid x)} \right)^{\alpha_k} = \pi_{\mathrm{ref}}(y\mid x)^{1-\sum_k\alpha_k} \prod_{k=1}^{n}\pi_k(y\mid x)^{\alpha_k}.
\label{eq:poe}
\end{equation}
Therefore, the optimal target policy is a product of experts in policy space, as in multi-objective decoding~\citep{Shi2024MOD}. Specifically, we can rewrite this in terms of \emph{policy log-ratio} for each basis policy, or how much it has moved away from the reference policy in terms of probability on each sequence
\begin{equation}
\phi_k(x,y) = \log \pi_k(y\mid x) - \log \pi_{\mathrm{ref}}(y\mid x),
\label{eq:logratio-feature}
\end{equation}
so that $\pi^*_{\mathrm{tar}}\propto\pi_{\mathrm{ref}}\exp(\sum_k\alpha_k\phi_k)$. When $\sum_k\alpha_k=1$, the explicit reference term vanishes.
Equation~\eqref{eq:poe} is exact under the idealized assumptions above. In practice, finite capacity models and imperfect optimization mean that a trained basis policy might not reach $\pi^*_{r_k}$. Our method therefore composes the \emph{implicit rewards actually represented by the trained policies}, rather than assuming that every basis perfectly maximizes its corresponding reward. We make this distinction explicit in Section~\ref{sec:method-regress}.

\subsection{Policy composition at inference time}
\label{sec:method-instantiations}

\subsubsection{Autoregressive language models}
\label{sec:method-llm}
For a language model, we obtain a response $y=(y_1,\ldots,y_T)$, given a given prompt $x$. Equation~\eqref{eq:poe} defines a distribution over complete responses, but sampling from that distribution exactly requires normalizers over the full response space. We obtain a practical decoder by applying the same composition at each prefix $h_t=(x,y_{<t})$, using the log-ratios for each token position $\phi_k(h_t,y_t)=\log\pi_k(y_t\mid h_t)-\log\pi_{\mathrm{ref}}(y_t\mid h_t)$, which sum to the sequence log-ratios, $\phi_k(x,y)=\sum_t\phi_k(h_t,y_t)$:
\begin{align}
\log \pi_{\methodm(\alpha)}(y_t\mid h_t) &\doteq \log \pi_{\mathrm{ref}}(y_t\mid h_t) + \sum_{k=1}^{n}\alpha_k\,\phi_k(h_t,y_t),
\label{eq:poe-decode}
\end{align}
where $\doteq$ denotes equality up to a normalizing constant that depends only on the prefix $h_t$. This decoder needs only the next token logits of the reference and basis policies.
This is not generally identical to the globally normalized sequence distribution in Eq.~\eqref{eq:poe}.
The two are similar in special cases in which these normalizers do not depend on the generated trajectory. In general, we treat Eq.~\eqref{eq:poe-decode} as the autoregressive approximation used by \ours.

\subsubsection{Diffusion models}
\label{sec:method-diffusion}
We now instantiate the framework for diffusion models~\citep{ho2020denoising, pmlr-v37-sohl-dickstein15}.
The goal in diffusion modeling is to sample from some distribution $p_0$. During training, we are given samples $X_0$ from $p_0$, we corrupt them by adding noise to form random variables $X_t = X_0 + \sigma(t) Z, \ Z \sim \mathcal N(0, I)$, for different noise levels $\sigma(t)$, and we train the model to reconstruct $X_0$ from $X_t$ with an $l_2$ loss. For a fixed noise level $t$, the optimal denoiser is the conditional expectation $\mathbb E[X_0 | X_t = \cdot, t]$. A model $h_{\theta^*}$ is trained to approximate this object.

We aim to connect the objects being trained, i.e. the conditional expectations, via Eq.~\eqref{eq:poe}.
Assuming $\sum_k \alpha_k = 1$, the posterior distribution can be composed as $\pi_{\mathrm{tar}}^*(x_0\mid X_t=x_t) \propto \prod_k \pi_k(x_0\mid X_t=x_t)^{\alpha_k}$, with normalizer $Z(x_t)$.
By taking logarithms and the gradient with respect to $x_t$ in both sides and using Tweedie's Formula~\citep{efron2011tweedie, tweedie1957statistical}, the final expression becomes (details in App.~\ref{app:diffusion-derivation}):
\begin{equation}
    \mathbb E_{r_{\mathrm{tar}}}[X_0 | X_t=x_t] = \sum_k \alpha_k \mathbb E_{r_{k}}[X_0 | X_t=x_t] + \sigma^2(t)\, \nabla \log Z(x_t).
    \label{eq:diffusion-combine}
\end{equation}
Simply put, Equation~\eqref{eq:diffusion-combine} states that the denoiser that we will get by adapting a diffusion model to a target reward $r_{\mathrm{tar}}$ is a linear combination of the existing denoisers that have been adapted to previous rewards, plus a correction term $\sigma^2(t)\nabla \log Z(x_t)$. Dropping it and linearly combining the predictions of the existing networks is an approximation to the target denoiser. This approximation still allows us to skip RL training altogether for the target reward.

\subsection{Recovering composition weights}
\label{sec:method-regress}

The composition method above requires composition weights $\alpha$. These may be specified directly by the user, but in our setup, only a target reward function is given. In this section, we present an algorithm that estimates these weights from a small dataset and access to the target reward function. We present the algorithm for the autoregressive models, but it naturally extends to the diffusion setup.

\myparagraph{The calibration set.} We fit the weights on a small calibration set of prompts $x_p$, each with $M_p$ responses $y_{p,m}$ sampled from the reference model or, if the basis policies are far from it, from the basis policies. A reward term that depends only on the prompt does not change the optimal policy, since $Z_r(x)$ in Eq.~\eqref{eq:boltz} absorbs it. As in DPO~\citep{rafailov2023direct}, we remove such terms by centering every reward and log-ratio within its prompt, $\widetilde r(x_p,y_{p,m})=r(x_p,y_{p,m})-\frac{1}{M_p}\sum_{m'}r(x_p,y_{p,m'})$, and likewise $\widetilde\phi_k$.

We provide two scenarios for finding composition weights. The first strategy requires the rewards the basis models were trained with, and the second one only requires the basis policies.

\myparagraph{Reward Composition Weights.} We score every prompt, response pair using the target reward and the basis rewards to obtain $\widetilde{\bm{r}}_{\mathrm{tar}}\in\mathbb{R}^{N}$, where $N=\sum_p M_p$, and the centered basis rewards into $\widetilde R\in\mathbb{R}^{N\times n}$, with columns $\widetilde r_k$. We estimate the reward space weights by ridge regression,
\begin{equation}
\bm{\alpha}^{\mathrm{R}} = \arg\min_{\alpha\in\mathbb{R}^{n}} \left\|\widetilde{\bm{r}}_{\mathrm{tar}}-\widetilde R\alpha\right\|_2^2 + \lambda\|\alpha\|_2^2.
\label{eq:reward-regression}
\end{equation}
When Eq.~\eqref{eq:linear-reward} holds, the calibration matrix has sufficient rank and $\lambda=0$, $\bm{\alpha}^{\mathrm{R}}$ recovers the true weights $\alpha^*$. Outside that setting, it gives the best regularized linear approximation of the target reward by basis rewards on the calibration distribution.

\myparagraph{Policy Composition Weights.} In the previous strategy, we require access not only to the basis policies but also to the reward functions that produced them. If those are not available, they can be approximated instead. In particular, we leverage the fact that each trained policy also defines an implicit reward, which depends on how much more probability the target policy assigns to a data point in comparison to the base policy. Rearranging Eq.~\eqref{eq:boltz} we have
\begin{equation}
r(x,y) = \beta \left[ \log \pi_r^*(y\mid x) - \log \pi_{\mathrm{ref}}(y\mid x) \right] + \beta\log Z_r(x).
\label{eq:implicit-reward}
\end{equation}
Applied to a trained basis policy, this identity makes $\beta\phi_k$ the reward for which $\pi_k$ is exactly optimal, up to the prompt normalizer, which centering removes. We can therefore use the policy log-ratios of Eq.~\eqref{eq:logratio-feature} as regression features. Let $\widetilde\Phi\in\mathbb{R}^{N\times n}$ be the feature matrix with columns $\widetilde\phi_k$, and estimate the policy space weights by
\begin{equation}
\widehat{\bm{\alpha}} = \arg\min_{\alpha\in\mathbb{R}^{n}} \left\|\widetilde{\bm{r}}_{\mathrm{tar}}-\beta\widetilde\Phi\alpha\right\|_2^2 + \lambda\|\alpha\|_2^2.
\label{eq:implicit-regression}
\end{equation}
The scale $\beta$ can be absorbed into $\alpha$ when the effective KL coefficient of the basis is unknown. Unlike $\bm{\alpha}^{\mathrm{R}}$, which describes how we can recover $r_{\mathrm{tar}}$ from rewards, $\widehat{\bm{\alpha}}$ describes how we can express it in terms of the \emph{directions of basis policies}. 

\subsection{Policy space coverage beyond linear rewards}
\label{sec:method-coverage}

So far, we have assumed that the target reward to be estimated can be expressed as a linear combination of the existing rewards. This exact reward composition is a sufficient condition for Eq.~\eqref{eq:poe} to hold, but it is not the only regime in which \ours can be useful. Specifically, the \ours policy is in the span of the log-ratios $\phi_1,\ldots,\phi_n$, regardless of any assumptions about the target reward. If the target policy that we would obtain through RL lies in that span, \ours can approximate it. We use the held-out $R^2$ of the target reward regressed on the policy log-ratios as a \emph{coverage} score.
After fitting policy composition weights $\widehat{\bm{\alpha}}$ on the fitting subset of the calibration set as explained above, we evaluate
\begin{equation}
\operatorname{Cov}_{\mathcal{B}}(r_{\mathrm{tar}}) = 1- \frac{ \left\| \widetilde{\bm{r}}_{\mathrm{tar}}^{\mathrm{val}} - \beta \widetilde\Phi^{\mathrm{val}}\widehat{\bm{\alpha}} \right\|_2^2 }{ \left\| \widetilde{\bm{r}}_{\mathrm{tar}}^{\mathrm{val}} \right\|_2^2 }.
\label{eq:coverage-score}
\end{equation}

For best results, the calibration distribution should consist of samples that we are interested in evaluating the approximated policy in. 
Naturally, if the target reward can be represented as a linear combination of basis policies the coverage score will be high. However, this metric can also be high for target rewards that are not linear combinations of the existing ones.
We can leverage \ours when coverage is high, and abstain or train a new basis policy when it is low.

This result formalizes the case in which \ours can generalize beyond reward combinations, as it does not rely on the reward linearity assumption. App.~\ref{app:coverage-bound} bounds the divergence to $\pi^*_{\mathrm{tar}}$ when the target reward is uniformly close to the span of the log-ratios.

\paragraph{\ours framework.} Given a target reward $r_{\mathrm{tar}}$, \ours 1) scores a small calibration set, 2) obtains weights using reward space regression in Eq.~\eqref{eq:reward-regression} or policy space regression in Eq.~\eqref{eq:implicit-regression}, 3) checks the held-out coverage score in Eq.~\eqref{eq:coverage-score}, and 4) composes the basis at inference time using Eq.~\eqref{eq:poe-decode} for language models or Eq.~\eqref{eq:diffusion-combine} for diffusion models. 

\section{Experimental Setup}
\label{sec:setup}

\begin{figure}[t]
\centering
\includegraphics[width=\linewidth]{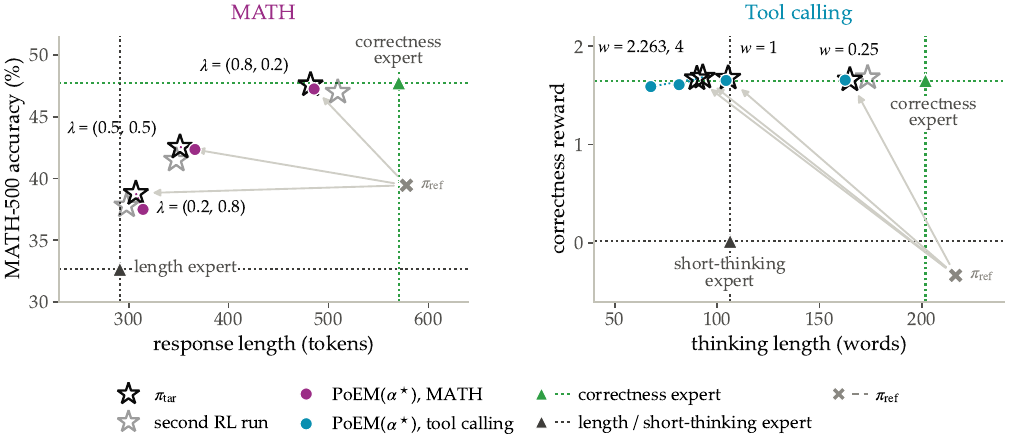}
\caption{\textbf{\ours predicts how RL changes each behavior of a combined reward.} \emph{Left}, Math: MATH-500 accuracy against mean response length for three targets $\pi_{\mathrm{tar}}$ trained on $\lambda_{\mathrm{corr}}r_{\mathrm{corr}}+\lambda_{\mathrm{len}}r_{\mathrm{len}}$ (stars), their second RL runs, and \ours$(\bm{\alpha}^*)$ at $\gamma_{\mathrm{geom}}$. \emph{Right}, tool calling: correctness reward against mean thinking length for targets trained on $r_{\mathrm{corr}}+w\,r_{\mathrm{short}}$ and \ours$(\bm{\alpha}^*)$ at $\gamma=1$. Each single expert (dotted lines) matches $\pi_{\mathrm{tar}}$ on at most one behavior, while \ours is close to $\pi_{\mathrm{tar}}$ on both. In both cases, \ours predicts both axes accurately.}
\label{fig:behavior}
\end{figure}

We test whether \ours predicts the target RL policy. For each target reward $r_{\mathrm{tar}}$ we train one policy $\pi_{\mathrm{tar}}$ with RL which is the policy \ours tries to approximate.

\subsection{Basis and target policies}
\label{sec:setup-bases}

We experiment with different types of reward functions for our basis across our language and diffusion experiments. We provide an overview below, and details of bases and experiments are in App.~\ref{app:setup}.

\textbf{Training setup.} Each language expert is trained on one reward from a shared initialization with LoRA~\citep{hu2021lora} from \texttt{Qwen3-0.6B}~\citep{yang2025qwen3}. Unless stated otherwise, experts and targets are trained on UltraChat prompts~\citep{ding2023enhancing}. Evaluation prompts come from the same datasets and are held out from every training set. For a combined reward, $\pi_{\mathrm{tar}}$ is trained on $r_{\mathrm{tar}}$ with the method and recipe of its basis. For a held-out reward, $\pi_{\mathrm{tar}}$ is the held-out expert itself.

\textbf{Programmatic rewards.} We train two sets of models on 20 different programmatic (verifiable) rewards. \pgrpo trains its experts on-policy with GRPO~\citep{shao2024deepseekmath}, with a KL penalty to $\pi_{\mathrm{ref}}$ in the loss, and \pdpo trains them offline with DPO, on pairs of fixed responses generated by the reference model, and ranked by the reward. Each reward is a deterministic function of the response text, such as vocabulary diversity, formality, readability (App.~\ref{app:rewards}). The combined rewards on these bases mix 2 to 16 $z$-scored basis rewards, with weights ranging from one dominant reward to uniform.

\myparagraph{Reward models.} Four bases are trained with PPO~\citep{schulman2017proximal} on reward models (RMs), with the KL penalty in the reward. \rmrb uses four RMs that rank high on RewardBench~\citep{lambert2025rewardbench}. \rmhh uses the helpfulness and harmlessness heads of \texttt{ArmoRM}. \rmpub uses ten diverse public RMs released by different groups (App.~\ref{app:rewards}). For these two bases, about half of the training prompts are harmful requests from PKU-SafeRLHF~\citep{ji2024pku}. \rmsky uses the five smaller RMs of the \texttt{Skywork-Reward-V2} family, to predict RL on \texttt{Skywork-Reward-V2-Llama-3.1-8B} (App.~\ref{app:large-rm}).

\myparagraph{MATH and tool calling.} Two more bases have two experts each, trained with GRPO and the KL penalty in the reward. On MATH~\citep{hendrycks2021measuring}, the experts are trained on answer correctness and on response length (App.~\ref{app:math}). On the tool calling dataset RLLA-4K~\citep{qian2026toolrl}, the experts are trained from \texttt{Qwen3-1.7B} on call correctness and on short thinking (App.~\ref{app:toolcall}).

\myparagraph{Image rewards.} \sdddpo has 13 LoRA experts on \texttt{Stable Diffusion v1.4}, trained with DDPO~\citep{black2023training}. Ten are trained on image statistics, such as compressibility, colorfulness and sharpness, and three on image RMs (App.~\ref{app:diffusion}).

\subsection{Evaluation and metrics}
\label{sec:setup-metrics}

\myparagraph{Composition strength.} We optionally multiply the basis contribution in Eq.~\eqref{eq:poe-decode} by a \emph{composition strength} $\gamma>0$, and write $\pi_{\methodm(\alpha;\gamma)}$ for the policy that decodes $\log\pi_{\mathrm{ref}}+\gamma\sum_k\alpha_k\phi_k$. At the sequence level, scaling by $\gamma$ is equivalent to replacing $\beta$ by $\beta/\gamma$. Technically, $\gamma$ is a test-time adjustment to the KL from the reference model.
Eq.~\eqref{eq:poe} holds at $\gamma=1$ when every policy optimizes Eq.~\eqref{eq:rlhf-obj} on its reward as given. However, the training recipes of \pgrpo and \pdpo discard the reward's scale, so on these bases $\pi_{\mathrm{tar}}$ moves about as far from $\pi_{\mathrm{ref}}$ as a single expert does, regardless of its reward. When the experts' policy log-ratios are weakly correlated, they partly cancel in their weighted average, so at $\gamma=1$ \ours stays closer to $\pi_{\mathrm{ref}}$ than $\pi_{\mathrm{tar}}$ does. We scale the composed log-ratio up by $\gamma_{\mathrm{geom}}$ (App.~\ref{app:eval-setup}) so that its norm on reference model responses matches the root mean square of the experts' norms, weighted by $|\alpha_k|$. We decode at $\gamma=1$, and on the programmatic bases also at $\gamma_{\mathrm{geom}}$.

\myparagraph{Weights and baselines.} For combined rewards on the programmatic bases, $\hat{\bm{\alpha}}$ is fitted on reference model responses. For held-out rewards and on the RM bases, it is fitted on responses sampled from the experts, without the held-out expert's responses. The baselines are best-of-$N$, the top expert and, on \rmpub, uniform weights.

\myparagraph{Coverage.} Before decoding for a held-out reward, we compute the coverage of Eq.~\eqref{eq:coverage-score} on a calibration set of responses sampled from the experts. A reward is \emph{covered} when its coverage is higher than a threshold (threshold = 0.3 in our experiments). On \rmpub we use \emph{within expert coverage}, which includes on-policy responses for the basis models. All its values fall below 0.3, so we split \rmpub at its median coverage instead.

\myparagraph{Metrics.} We score a policy $\pi$ by its mean target reward $r(\pi)$ and report the share of $\pi_{\mathrm{tar}}$'s reward gain that it recovers, $\mathrm{rec}(\pi)=\big(r(\pi)-r(\pi_{\mathrm{ref}})\big)/\big(r(\pi_{\mathrm{tar}})-r(\pi_{\mathrm{ref}})\big)$, so that $\pi_{\mathrm{ref}}$ scores 0 and $\pi_{\mathrm{tar}}$ scores 1. The reward error is $|1-\mathrm{rec}(\pi)|$. A policy can reach the target reward without behaving like $\pi_{\mathrm{tar}}$, so we also measure the per-token $\mathrm{KL}(\pi_{\mathrm{tar}}\,\|\,\pi)$ along $\pi_{\mathrm{tar}}$'s own greedy responses. The relative KL divides it by $\mathrm{KL}(\pi_{\mathrm{tar}}\,\|\,\pi_{\mathrm{ref}})$, so a relative KL below 1 means the prediction is closer to $\pi_{\mathrm{tar}}$ than the reference model is. Except on MATH and tool calling, which sample at temperature 1, all language model policies, including \ours, decode greedily on held-out prompts. Each decoded policy is scored against the $\pi_{\mathrm{ref}}$ and $\pi_{\mathrm{tar}}$ decoded in the same run, and we report medians over targets, with 95\% bootstrap intervals on the programmatic bases. App.~\ref{app:setup} includes all details about evaluation and metrics.

\section{Results}
\label{sec:experiments}

\begin{figure}[t]
    \centering
    \includegraphics[width=\linewidth]{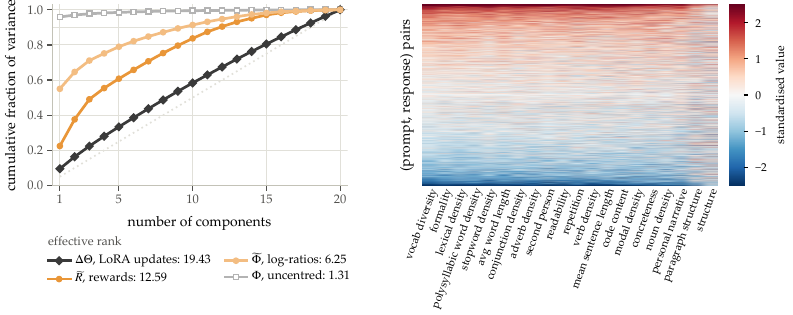}
    \caption{\textbf{Policies trained on different rewards vary along fewer directions than the rewards.} We trained 20 models on different programmatic rewards (\pgrpo), and evaluated each on model responses. \emph{Left:} Cumulative variance of the experts' weight updates ($\Delta\Theta$), rewards ($\widetilde R$) and log-ratios ($\widetilde\Phi$). The weight updates are close to full rank, while the log-ratios use about half as many effective directions as the rewards. \emph{Right:} The log-ratio matrix $\widetilde\Phi$, that we computed the rank for. 
    }
    \label{fig:rank-grpo}
\end{figure}

We first test \ours on combined rewards, where our theory applies, then study the geometry of the post-trained experts and test \ours on held-out rewards and on image diffusion.

\subsection{Predicting RL outcomes for combined rewards}
\label{sec:linear-setting}

We first evaluate whether \ours reproduces the outcome of RL on combined rewards as in Eq.~\eqref{eq:linear-reward}, where the exact composition rule applies under idealized assumptions. We compare both the reward achieved by \ours and its similarity to the trained target policy.

\paragraph{\ours predicts how RL trades off different behaviors.} We examine whether \ours recovers \emph{the tradeoff learned by RL} for different rewards, rather than simply achieving a high combined reward (Fig.~\ref{fig:behavior}). Across the MATH targets (App.~\ref{app:math}), \ours predicts both behaviors: its accuracy is within 1.3 percentage points of the target RL policy, and its mean response length differs by at most 15 tokens. In comparison, the best single expert can differ by up to 10 accuracy points or 88 tokens.
We observe a similar result on tool calling (App.~\ref{app:toolcall}). At the smallest short-thinking weight, $w=0.25$, the target RL policy and \ours learn to produce reasoning of a similar length. 
At higher short-thinking weights, both \ours and target RL learn to produce much shorter reasoning, while maintaining similar tool call correctness.

\begin{table}[t]
\centering
\caption{\textbf{On combined rewards, \ours performs similarly to the target policies trained with RL.} Reward error is 0 when a method matches $\pi_{\mathrm{tar}}$'s reward gain and 1 when it is as far off as the reference model. A second RL run with another seed shows how much RL itself varies. The last two columns measure if \ours is distributionally close to $\pi_{\mathrm{tar}}$: its KL to $\pi_{\mathrm{tar}}$ as a fraction of the reference model's, and how often it is the closer of the two. We report medians over combined rewards, with \ours at $\gamma_{\mathrm{geom}}$ on the programmatic bases and $\gamma=1$ on the RM bases.}
\label{tab:combined-summary}
\small
\setlength{\tabcolsep}{3.6pt}
\begin{tabular}{@{}lcccccccc@{}}
\toprule
 & & \multicolumn{5}{c}{Reward error $\downarrow$} & \multicolumn{2}{c}{Distance of \ours$(\bm{\alpha}^*)$} \\
\cmidrule(lr){3-7}\cmidrule(l){8-9}
Basis & Targets & \ours$(\hat{\bm{\alpha}})$ & \ours$(\bm{\alpha}^*)$ & top expert & best-of-16 & 2nd RL run & relative KL & closer \\
\midrule
\pgrpo & 32 & 0.28 & 0.19 & 0.23 & 0.37 & 0.12 & 0.24 & 32/32 \\
\pdpo & 32 & 0.19 & 0.21 & 0.32 & 0.40 & 0.17 & 1.05 & 13/32 \\
\rmrb & 8 & 0.05 & 0.08 & 0.13 & 0.24 & -- & 0.29 & 8/8 \\
\rmhh & 3 & 0.07 & 0.22 & 0.07 & 0.22 & -- & 0.58 & 3/3 \\
\bottomrule
\end{tabular}
\end{table}

\paragraph{\ours recovers RL's reward gain.}
We next evaluate combined rewards over four more sets of experts trained using GRPO, DPO, and PPO. Specifically, we consider multiple combined targets for each basis, compare \ours and target policies, and summarize the results in Table~\ref{tab:combined-summary}.
On the RM bases, \ours at $\gamma=1$ already recovers 0.78 to 1.08 of $\pi_{\mathrm{tar}}$'s reward gain. On the programmatic bases, $\gamma=1$ underestimates the strength of the target policy, especially when combining more rewards (App.~\ref{app:combined-prog-rm}).
The strength hyperparameter at $\gamma_{\mathrm{geom}}$ estimates how much policy log-ratios shrink in comparison to the experts, and after applying it, recovery increases to 0.83 to 1.00. We also find that \ours recovery is close to RL's own variation. 

\paragraph{\ours is distributionally close to the target policy.} 
Obtaining RL's reward does not yet mean we have predicted $\pi_{\mathrm{tar}}$. 
On \pgrpo and the RM bases, \ours does both, as it is closer in KL divergence to $\pi_{\mathrm{tar}}$ than the reference model on every combined reward.
\pdpo is an exception, where \ours matches the target reward after geometric rescaling, but does not consistently recover its policy distribution. This suggests that when the basis policies are trained offline, they might not be able to reproduce the outcome of the target policy, while they can achieve the same reward.
Unlike best-of-$N$, \ours does not need to query $r_{\mathrm{tar}}$ at decoding. App.~\ref{app:combined-prog-rm} includes more baseline comparisons. 

Overall, \ours can approximate both the reward and behavioral tradeoffs of RL training on combined rewards. Across the on-policy bases, the composed policies are also distributionally close to the trained target policies.

\subsection{Beyond the linear case: policy space geometry of post-trained models}
\label{sec:geometry}
Theoretically, we relied on the strong assumption that the new reward is exactly a linear combination of the basis rewards, $r_{\mathrm{tar}} = \sum_k \alpha_k\,r_k$. We now explore what happens when this assumption does not hold. 
\ours composes the basis policy \emph{log-ratios} $\log\pi_k - \log\pi_{\mathrm{ref}}$, rather than their rewards. A relevant question is therefore whether the target policy log-ratio can lie near the subspace spanned by the basis log-ratios, even when the target reward is not a linear combination of the basis rewards. A \textit{sufficient} condition for that to happen is when $r_{\mathrm{tar}}$ lies in the linear span of the basis rewards. We show here that this condition is \textbf{not} \textit{necessary}.

\paragraph{RL policies vary along fewer directions than their rewards.}
We examine the policy space, weight updates, and rewards of RL models from the \pgrpo basis. Specifically, we measure the effective rank of the policy log-ratio matrix $\widetilde\Phi$ from Section~\ref{sec:method-regress} across $n=20$ basis models. The effective rank of a matrix is $\exp(-\sum_i p_i\log p_i)$, where $p_i$ is the share of variance along its $i$-th principal direction.
We make the surprising experimental finding that the (approximate) rank of this matrix is smaller than $n$, despite nearly orthogonal parameter updates, and higher rank basis rewards (Fig.~\ref{fig:rank-grpo}). 
Because $\widetilde\Phi$ is approximately low-rank with effective dimension much smaller than $n$, the set of policy log-ratios reachable by $\sum_k \alpha_k(\log\pi_k - \log\pi_{\mathrm{ref}})$ is a low-dimensional space that the basis covers densely.

\paragraph{Held-out policies can be easier to represent than their rewards.}
We observe a similar phenomenon with experts trained on diverse reward models. On \rmpub, the other experts' log-ratios explain more variation of a held-out expert's log-ratio than the other rewards explain of its reward, for all ten experts  (App.~\ref{app:rmdiv-heldout}). This provides further evidence that policy changes may be easier to represent using existing experts than the rewards that induce them. 

Together, these observations suggest that experts trained on different rewards change the policy along directions that can be reused for a new reward.

\subsection{Predicting RL outcomes for held-out rewards}
\label{sec:new-rewards}

The previous section suggests that \ours can approximate the target policy of a given reward even when the reward is not a weighted sum of the basis rewards, as long as $\pi_{\mathrm{tar}}$'s log-ratio lies near the span of the experts' log-ratios. We now test this on held-out rewards, where composition weights are estimated without access to the target policy.

\begin{figure}[t]
\centering
\begin{minipage}[c]{0.49\linewidth}
\centering
\includegraphics[width=\linewidth]{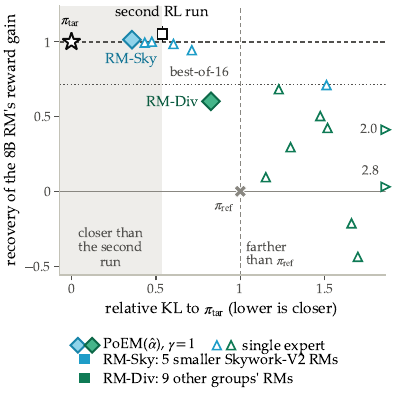}
\end{minipage}\hfill
\begin{minipage}[c]{0.48\linewidth}
\caption{\textbf{Experts trained with RL on smaller reward models predict RL on a larger one.} $\pi_{\mathrm{tar}}$ is RL on the reward model \texttt{Skywork-Reward-V2-Llama-3.1-8B}, and all policies share the \texttt{0.6B} reference model. Diamonds show \ours$(\hat{\bm{\alpha}})$ at $\gamma=1$ over \rmsky, five experts trained on the smaller RMs of the same family, or over the nine other experts of \rmpub, and triangles the single experts. Over \rmsky, \ours matches $\pi_{\mathrm{tar}}$'s reward and is closer to it than a second RL run (square, shaded region). Over \rmpub, it is closer to $\pi_{\mathrm{tar}}$ than $\pi_{\mathrm{ref}}$ and every one of its experts. The dotted line is best-of-16 from $\pi_{\mathrm{ref}}$ reranked by the 8B RM.}
\label{fig:large-rm-main}
\end{minipage}
\end{figure}

\paragraph{\ours predicts RL on a larger reward model.} We first consider a practical setting where RL on a large RM is expensive, since the RM must score every rollout. We predict RL on an 8B RM from experts trained with RL on smaller RMs, using the 8B RM only to score the calibration set (App.~\ref{app:large-rm}). Over \rmsky, whose experts are trained on smaller RMs of the same family, \ours matches $\pi_{\mathrm{tar}}$'s reward gain and is closer to $\pi_{\mathrm{tar}}$ than a second RL run on the 8B RM (Fig.~\ref{fig:large-rm-main}). These experts also reach most of the reward gain on their own, but each is farther from $\pi_{\mathrm{tar}}$ than their composition. Over the nine other experts of \rmpub, trained on RMs from other groups, \ours recovers less of the reward gain but is still closer to $\pi_{\mathrm{tar}}$ than any of these experts alone.

\begin{figure}[t]
\centering
\includegraphics[width=\linewidth]{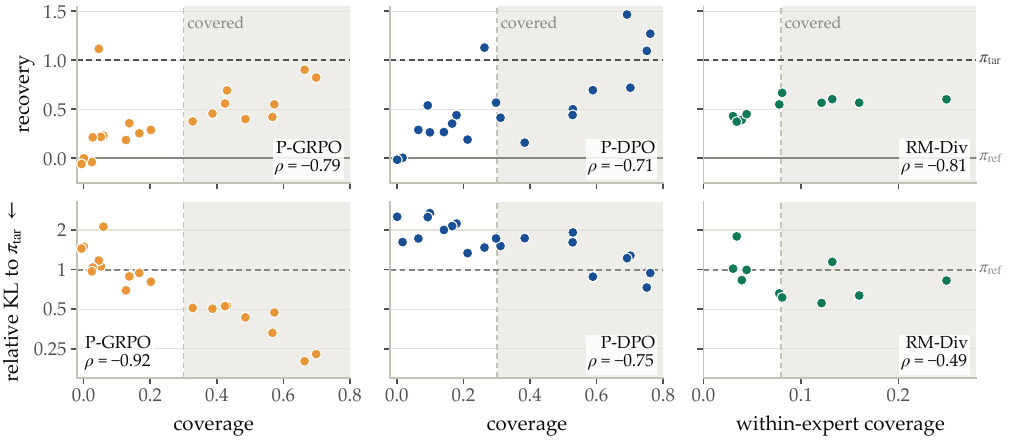}
\caption{\textbf{Coverage, computed before decoding, predicts both how much of a held-out reward \ours recovers and how close it gets to $\pi_{\mathrm{tar}}$.} Each point is one expert, held out and predicted from the others. \emph{Top:} recovery of its reward. \emph{Bottom:} KL to $\pi_{\mathrm{tar}}$ relative to that of the reference model, where values below 1 are closer. The shaded region is covered (on \rmpub, above the median within-expert coverage). $\rho$ is the Spearman correlation of coverage with the reward error (top) and with the relative KL (bottom).
We find that with higher coverage, the achieved recovery increases, and relative KL decreases.
}
\vspace{-.5em}
\label{fig:heldout}
\end{figure}

\paragraph{Coverage predicts which new rewards \ours can reach.} To see when such predictions work, we hold out each expert of \pgrpo, \pdpo and \rmpub in turn and predict it from the others. On all three bases, coverage ranks the reward error, and covered rewards recover more of RL's gain than the others (Fig.~\ref{fig:heldout}, Table~\ref{tab:heldout-summary}). A high coverage means that the experts already span the direction in which RL moves the policy. Held-out rewards remain harder than combined ones, so coverage is most useful as a filter on which predictions to use. On these rewards it rarely passes a failure: only 3 of the 23 covered predictions, all on \pdpo, neither recover half of RL's gain nor end closer to $\pi_{\mathrm{tar}}$ than the reference model, against 15 of the 27 uncovered ones. It ranks rewards within a basis but is not a threshold across bases: on \rmrb, the held-out rewards have coverage near zero yet recover more than half of RL's gain, because its reward models largely agree (App.~\ref{app:heldout-rm}).

\paragraph{With on-policy experts, covered predictions also distributionally close to $\pi_{\mathrm{tar}}$.} On \pgrpo, all covered predictions are closer to $\pi_{\mathrm{tar}}$ than the reference model, and coverage ranks this distance even better than the reward error (Fig.~\ref{fig:heldout}, bottom). On \rmpub, most better-covered predictions are closer. On \pdpo, as for combined rewards, covered predictions reach the reward but not the distance.
On \rmpub, composing also beats decoding the top expert alone on most held-out rewards, both in reward and in distance to the held-out expert (App.~\ref{app:heldout-rm}). 

\paragraph{Increasing composition strength cannot compensate for missing directions.} 
We also study whether poor recovery on low-coverage rewards can be corrected by increasing the composition strength $\gamma$ (App.~\ref{app:coverage-strength}). Increasing $\gamma$ improves recovery for covered rewards, while uncovered rewards gain less and can move farther from the target policy. Even $\gamma$ tuned on $\pi_{\mathrm{tar}}$ for each reward does not close the gap between uncovered rewards and covered ones.

\subsection{Image diffusion models: qualitative results}
\label{sec:diffusion}

We also test whether diffusion models post-trained with RL can be reconstructed from experts trained on other rewards. We fit the weights using the held-out expert's noise predictions at each denoising step. A similar geometry holds in image diffusion models. Figure~\ref{fig:diffusion-generations} shows that the composition reproduces the change the target model makes, such as more saturated colors, finer texture, and a centered subject on a blurred background. Unlike for language models, the weights are fit separately at each denoising step. Our finding is that the held-out expert's effect lies largely within what the other experts can express (App.~\ref{app:diffusion}), meaning that linear combinations of the conditional expectations of the policies in the basis can reproduce much of the target denoiser's effect. This shows that the effects of RL fine-tuning can be partly represented by existing experts, although predicting them from the reward alone remains open.
\vspace{-0.5em}

\begin{figure}[t]
    \centering
    \includegraphics[width=\linewidth]{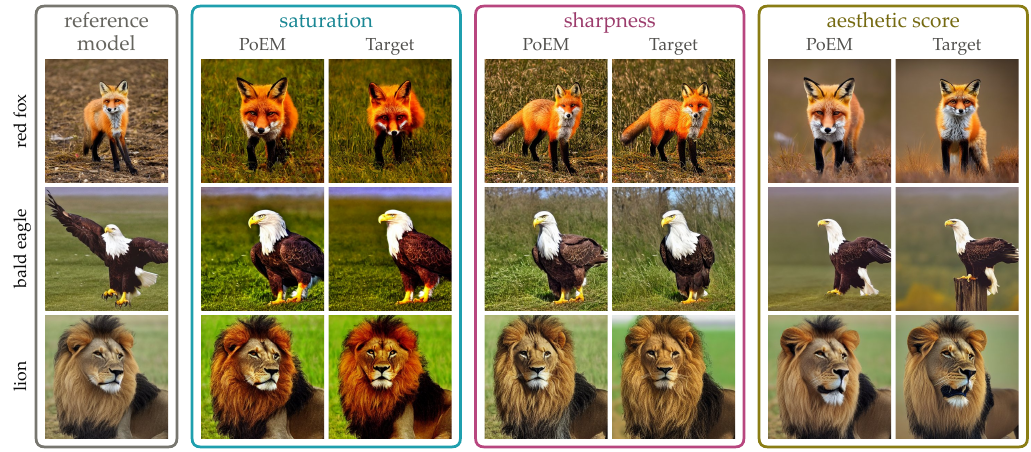}
    \caption{\textbf{Reconstructing held-out diffusion models post-trained with RL.} For each of three rewards, \ours is composed from the experts trained on the other rewards, and Target is the model fine-tuned on that reward with DDPO. The first column shows the reference model.}
    \vspace{-.5em}
    \label{fig:diffusion-generations}
\end{figure}

\section{Related Work}
\label{sec:related}

\myparagraph{Policy composition.}
Policies trained for different objectives can be combined in parameter space, by weighted sums of their parameters~\citep{rame2023rewarded,jang2023personalized,xie2025bone} or mixtures of LoRA adapters~\citep{feng2024mixture}, or at decoding time, by combining their logits~\citep{liu2021dexperts,liu2024tuning}. Closest to our work, MOD~\citep{Shi2024MOD} combines policies trained on single objectives with the rule of \ours (for the reverse KL and weights that sum to one), and DeRa~\citep{liu2024decoding} uses this rule with one aligned model to change the strength of the KL penalty. Concurrently with our work, PRISM~\citep{liang2026don} composes one learned prefix per reward on a shared model. In these works, the user sets the weights or picks the best of many on scored samples. \ours fits them by regression on a small calibration set scored by the new reward, and we compare the composed policy with the one RL learns on that reward.

\myparagraph{Multi-objective post-training.}
Another line of work trains one model on several objectives at once, with tradeoffs that are fixed during training or chosen at inference. Examples include reward models with one output per objective~\citep{wang2024interpretable}, DPO with several objectives~\citep{zhou2023beyond}, policies conditioned on target rewards~\citep{yang2024rewards} or on preference weights~\citep{gupta2025robust}, a primary objective maximized under lower bounds on the others~\citep{agnihotri2025multi}, and GRPO with each reward normalized separately~\citep{liu2026gdpo}. \ours trains each expert on a single reward, and we run RL on the combined reward only to evaluate the prediction.

\myparagraph{Implicit rewards and transfer.}
\citet{rafailov2023direct} used the closed form of the optimum of RL with a KL penalty~\citep{peng2019advantage,korbak2022rl} to show that a policy's log-ratio to its reference model is an implicit reward. This implicit reward has been used to score responses for iterative DPO~\citep{chen2024bootstrapping} and as a process reward model~\citep{cui2025process}. \ours instead regresses a new reward on the implicit rewards of many experts. This is similar to successor features in RL~\citep{barreto2016successor}, which regress the reward of a new task on a fixed set of features, evaluate earlier policies on that task, and combine them into a policy for it. In \ours, the features are the experts' log-ratios, learned by previous RL runs, so the target reward need not be linear in the rewards the experts were trained on.

\section{Discussion}

\myparagraph{What RL post-training changes.} The experts' weight updates are nearly orthogonal, yet in policy space, they move along fewer directions than the rewards that trained them. This suggests that RL training may move a reference model along a number of directions regardless of its reward, the reference model may offer only a few directions that are easy to amplify, or different rewards may share useful behaviors. This is in line with prior works that show post-training mostly increases the probability of trajectories the reference model can already sample~\citep{yue2026does}, and that even spurious rewards can improve math reasoning in some models~\citep{shao2025spurious}. Our results point to a related and testable statement: a few policy directions learned for some rewards can predict much of what RL learns for others, when the new reward moves the policy in a direction the basis already has.

\myparagraph{Previewing a reward before training on it.} \ours obtains a policy without training, so one can see what a candidate reward would do before starting an RL run, and trust the preview only when coverage is high. The most useful preview is of effects the reward does not measure, such as changes in unrelated capabilities. We have not tested whether \ours predicts unintended effects, such as rare failures that the calibration samples miss.

\myparagraph{Approximating expensive rewards.} RL evaluates the target reward on every rollout, while \ours evaluates it only on a small calibration set. So RL on a reward that is expensive to evaluate, such as a large reward model, an LLM judge or a human rater, could be approximated with experts trained on cheap rewards. 
Our experiment on predicting RL on a larger reward model from experts trained on smaller ones demonstrates this possibility. More broadly, this could enable faster iteration over candidate rewards without retraining for each one.

\myparagraph{Building reusable bases.} Coverage shows which target rewards the current basis cannot reach, so the next expert could target the reward with the lowest coverage. Running $k$ experts at inference is also costly, which matters more for deploying a policy than for predicting one. Distilling \ours into a single model, and starting RL from it so that training only adds the missing directions, would remove this cost. Our experiments also keep the training prompts fixed and vary only the reward. We have not tested whether policies trained on different prompt sets can be composed to predict RL on a new prompt distribution.

\section{Conclusion}
We studied whether the outcome of RL on a new reward can be predicted from policies already trained on other rewards. When the new reward is a linear combination of the old ones, composing the experts with the combination weights, at a strength set without any new training, predicts the RL outcome's reward about as well as a second RL run does. When the new reward is not a combination, a coverage score computed from the experts ranks in advance which rewards can be reached by our basis. Rewards with low coverage stay far from the RL outcome at every $\gamma$ we tried. With experts trained on reward models, \ours is closer to the RL outcome than the reference model and the top expert on every combined reward, and coverage again ranks the held-out RMs. Together, these results suggest that a set of trained RL policies carries enough structure to compose new reward aligned behavior at inference time.

\section*{Limitations}
\label{sec:limitations}

The limitation of \ours in predicting new RL outcomes relies on whether the target policy's changes are captured in the existing basis. Our coverage score estimates this based on reward variation on calibration samples. However, it is possible that the target policy lies outside the basis span. New rewards may cause exploitation of rare rewards which may be outside the calibration distribution, which might cause in failures of prediction despite coverage.

\ours makes several assumptions about the RL setup it operates in, which may not always be true. Our exact composition rule assumes that, in the limit of optimality, multiple policies share a common reference model and are regularized by a KL divergence. RL, in practice, may use no KL regularization, and/or may not converge to optimality. Also, \ours uses an autoregressive decoder to approximate composition at the sequence level by conditioning on previous actions. While current implementations of \ours address these assumptions, it is unknown how they affect prediction. We have not explored how \ours would scale to address larger scale models, and longer training for RL.

\section*{Acknowledgements}
We thank Yoon Kim, Mehul Damani, and Idan Shenfeld, and members of the Torralba lab for helpful discussions, as well as Sunny Qin and Sheng-Yu Wang for feedback on figures. This work was supported by ONR MURI grant N00014-26-1-2255. Giannis Daras is additionally supported by a Jane Street Research Collaboration award, Google's TPU Research Cloud (TRC) program, and Lambda's Research Grant Program.

% \newpage

\bibliographystyle{plainnat}
\bibliography{bib}

%%%%%%%%%%%%%%%%%%%%%%%%%%%%%%%%%%%%%%%%%%%%%%%%%%%%%%%%%%%%

\newpage
\appendix
\section{Additional Theoretical Results and Details}
\label{app:theory}

\subsection{A guarantee under coverage}
\label{app:coverage-bound}

Coverage (Section~\ref{sec:method-coverage}) explains whether the target reward is close to a combination of the experts' log-ratios on the calibration set. A uniform version of this condition bounds how far the sequence-level product can be from $\pi^*_{\mathrm{tar}}$. For any weights $\alpha$, define the implicit composed reward
\begin{equation}
\widehat r_{\alpha}(x,y) = \beta\sum_{k=1}^{n}\alpha_k\,\phi_k(x,y).
\label{eq:composed-implicit-reward}
\end{equation}
The sequence-level product of Eq.~\eqref{eq:poe} with these weights, $\pi_{\mathrm{seq}}\propto\pi_{\mathrm{ref}}\exp(\sum_k\alpha_k\phi_k)$, is exactly the optimizer of $J_{\widehat r_{\alpha}}$, even when the experts are not optimal for their basis rewards. Moreover, for any reward $r$ and policy $\pi$, the KL regularized objective satisfies the identity
\begin{equation}
J_r(\pi_r^*)-J_r(\pi) = \beta\, \mathbb{E}_{x\sim\mathcal{D}} \left[ \mathrm{KL}\!\left( \pi(\cdot\mid x)\,\|\,\pi_r^*(\cdot\mid x) \right) \right].
\label{eq:regularized-gap-identity}
\end{equation}

\begin{proposition}
If there exist weights $\alpha$ and a prompt-only function $b(x)$ such that
\begin{equation}
\sup_{x,y} \left| r_{\mathrm{tar}}(x,y) - \widehat r_{\alpha}(x,y) - b(x) \right| \leq \varepsilon,
\label{eq:coverage-assumption}
\end{equation}
then the sequence-level product $\pi_{\mathrm{seq}}$ with these weights obeys
\begin{align}
J_{r_{\mathrm{tar}}}(\pi^*_{\mathrm{tar}}) - J_{r_{\mathrm{tar}}}(\pi_{\mathrm{seq}}) &\leq 2\varepsilon,
\label{eq:coverage-objective-bound}\\
\mathbb{E}_{x\sim\mathcal{D}} \left[ \mathrm{KL}\!\left( \pi_{\mathrm{seq}}(\cdot\mid x) \,\|\, \pi^*_{\mathrm{tar}}(\cdot\mid x) \right) \right] &\leq \frac{2\varepsilon}{\beta}.
\label{eq:coverage-kl-bound}
\end{align}
\end{proposition}

\begin{proof}
Adding $b(x)$ to a reward shifts $J$ by the same constant for every policy, so $\pi_{\mathrm{seq}}$ also maximizes $J_{\widehat r_{\alpha}+b}$. Replacing a reward by a uniformly $\varepsilon$-close one changes the value of any policy by at most $\varepsilon$. Hence
\begin{equation*}
J_{r_{\mathrm{tar}}}(\pi^*_{\mathrm{tar}}) \leq J_{\widehat r_{\alpha}+b}(\pi^*_{\mathrm{tar}})+\varepsilon \leq J_{\widehat r_{\alpha}+b}(\pi_{\mathrm{seq}})+\varepsilon \leq J_{r_{\mathrm{tar}}}(\pi_{\mathrm{seq}})+2\varepsilon,
\end{equation*}
which is Eq.~\eqref{eq:coverage-objective-bound}. Eq.~\eqref{eq:coverage-kl-bound} then follows from Eq.~\eqref{eq:regularized-gap-identity} with $r=r_{\mathrm{tar}}$ and $\pi=\pi_{\mathrm{seq}}$.
\end{proof}

\paragraph{Which responses coverage should be measured on.} Coverage is an average over calibration responses, whereas Eq.~\eqref{eq:coverage-assumption} bounds the residual on every response. The following identity makes precise which average is the relevant one. Write $\delta(x,y)=r_{\mathrm{tar}}(x,y)-\widehat r_{\alpha}(x,y)$ for the residual. Because $\pi^*_{\mathrm{tar}}(y\mid x)\propto\pi_{\mathrm{seq}}(y\mid x)\exp(\delta(x,y)/\beta)$, we have for every prompt $x$
\begin{align}
\mathrm{KL}\!\left(\pi_{\mathrm{seq}}(\cdot\mid x)\,\|\,\pi^*_{\mathrm{tar}}(\cdot\mid x)\right)
&= \log\mathbb{E}_{y\sim\pi_{\mathrm{seq}}(\cdot\mid x)}\!\left[e^{\delta(x,y)/\beta}\right] - \frac{1}{\beta}\,\mathbb{E}_{y\sim\pi_{\mathrm{seq}}(\cdot\mid x)}\!\left[\delta(x,y)\right] \nonumber\\
&= \frac{1}{2\beta^2}\operatorname{Var}_{y\sim\pi_{\mathrm{seq}}(\cdot\mid x)}\!\left[\delta(x,y)\right] + O\!\left(\Delta(x)^3/\beta^3\right),
\label{eq:kl-residual-variance}
\end{align}
where $\Delta(x)=\max_y\delta(x,y)-\min_y\delta(x,y)$ is the spread of the residual over the responses to prompt $x$. As in the centered regression of Section~\ref{sec:method-regress}, a prompt-only term in $\delta$ does not change the KL, since it shifts $\log\mathbb{E}[e^{\delta/\beta}]$ and $\mathbb{E}[\delta]/\beta$ by the same amount. Under Eq.~\eqref{eq:coverage-assumption}, every $\delta(x,y)$ lies within $\varepsilon$ of $b(x)$, so $\Delta(x)\leq 2\varepsilon$. Hoeffding's lemma then bounds the exact KL on every prompt by $\Delta(x)^2/(8\beta^2)\leq\varepsilon^2/(2\beta^2)$. The average over prompts obeys the same bound, which is tighter than Eq.~\eqref{eq:coverage-kl-bound} whenever $\varepsilon<4\beta$. Away from this worst case, Eq.~\eqref{eq:kl-residual-variance} shows that to second order the KL depends only on the within-prompt variance of the residual, on responses that $\pi_{\mathrm{seq}}$ generates. One minus coverage estimates this variance (relative to that of $r_{\mathrm{tar}}$) as long as the calibration responses look like those responses. This is why, for held-out rewards, we sample the calibration responses from the experts. The second-order expansion can fail badly in one case: a response with a large residual that $\pi_{\mathrm{seq}}$ sometimes produces but that is missing from the calibration set, e.g., a rare exploit of $r_{\mathrm{tar}}$. It can dominate $\log\mathbb{E}[e^{\delta/\beta}]$ without changing coverage at all.

The bound and Eq.~\eqref{eq:kl-residual-variance} hold for the sequence-level product, not for the decoder of Eq.~\eqref{eq:poe-decode}. App.~\ref{app:decoder-gap} measures the gap between the two.

\subsection{Composition for diffusion models}
\label{app:diffusion-derivation}
Section~\ref{sec:method-diffusion} composes the posterior, for $\sum_k\alpha_k=1$, as
\begin{equation}
\pi_{\mathrm{tar}}^*(x_0\mid X_t=x_t) = \frac{1}{Z(x_t)} \prod_k \pi_k(x_0\mid X_t=x_t)^{\alpha_k}, \quad Z(x_t) = \int \prod_k \pi_k(x_0\mid X_t=x_t)^{\alpha_k}\, dx_0 .
\label{eq:diffusion-posterior}
\end{equation}
By taking logarithms and the gradient with respect to $x_t$ in both sides of Eq.~\eqref{eq:diffusion-posterior}, we have that
\begin{equation}
    \nabla \log \pi_{\mathrm{tar}}^*(x_0\mid X_t=x_t) = \sum_{k} \alpha_k \nabla \log \pi_k(x_0\mid X_t=x_t) - \nabla \log Z(x_t).
    \label{eq:diffusion-cond-scores}
\end{equation}
We now have to work with these conditional scores. An application of Bayes formula gives $\nabla \log \pi_{k}(x_0\mid X_t=x_t) = \underbrace{\nabla \log \pi_{k}(x_t\mid X_0=x_0)}_{A} - \nabla \log \pi_{k}(x_t)$. The important observation is that with $x_0$ fixed, the first likelihood term is Gaussian independent of the policy $\pi_k$. Hence, Eq.~\eqref{eq:diffusion-cond-scores} becomes $A - \nabla \log \pi^*_{\mathrm{tar}}(x_t) = \sum_k \alpha_k (A - \nabla \log \pi_k(x_t)) - \nabla \log Z(x_t) \iff \nabla \log \pi^*_{\mathrm{tar}}(x_t) = \sum_k \alpha_k \nabla \log \pi_k(x_t) + \nabla \log Z(x_t)$. For the last step of this calculation, we are invoking a powerful statistical tool, called Tweedie's Formula~\citep{efron2011tweedie, tweedie1957statistical}, $\mathbb E[X_0 \mid X_t=x_t] = x_t + \sigma^2(t)\nabla \log \pi(x_t)$, that connects the gradient of the log-likelihood (also known as the score) with the conditional expectation the model is trained to estimate. The final expression is Eq.~\eqref{eq:diffusion-combine}.
By H\"older's inequality, $Z(x_t) \le 1$, with equality if and only if all posteriors $\pi_k(\cdot \mid X_t = x_t)$ coincide. The correction vanishes as $\sigma(t)\to 0$, where all posteriors concentrate at $x_t$, and it is exactly zero when, e.g., the basis policies are Gaussians with a shared covariance, since then $Z(x_t)$ does not depend on $x_t$. In general, however, it is non-zero.

\section{Experimental details}
\label{app:setup}
Table~\ref{tab:bases} lists the bases used in our experiments.

\begin{table}[H]
\centering
\caption{\textbf{The bases and the experiments run on each.} The last three columns show the number of target rewards and the section that reports them.}
\label{tab:bases}
\small
\setlength{\tabcolsep}{4pt}
\begin{tabular}{@{}lllccc@{}}
\toprule
Basis & Experts trained on & RL (KL term) & Geometry & Combined & Held-out \\
\midrule
\pgrpo & 20 programmatic rewards & GRPO (loss) & Sec.~\ref{sec:geometry} & 32, Sec.~\ref{sec:linear-setting} & 20, Sec.~\ref{sec:new-rewards} \\
\pdpo & the same 20 rewards & DPO ($\beta$) & App.~\ref{app:rank-other-bases} & 32, Sec.~\ref{sec:linear-setting} & 20, Sec.~\ref{sec:new-rewards} \\
\rmrb & 4 RewardBench RMs & PPO (reward) & App.~\ref{app:rank-other-bases} & 8, Sec.~\ref{sec:linear-setting} & 4, App.~\ref{app:heldout-rm} \\
\rmhh & 2 \texttt{ArmoRM} heads & PPO (reward) & -- & 3, Sec.~\ref{sec:linear-setting} & -- \\
\rmpub & 10 public RMs & PPO (reward) & App.~\ref{app:rank-other-bases} & -- & 10, Sec.~\ref{sec:new-rewards} \\
\rmsky & 5 \texttt{Skywork-Reward-V2} RMs & PPO (reward) & -- & -- & 1, Sec.~\ref{sec:new-rewards} \\
MATH & correctness, length & GRPO (reward) & -- & 3, App.~\ref{app:math} & -- \\
Tool calling & correctness, short thinking & GRPO (reward) & -- & 4, App.~\ref{app:toolcall} & -- \\
\sdddpo & 13 image rewards & DDPO (none) & -- & -- & 13, App.~\ref{app:diffusion} \\
\bottomrule
\end{tabular}
\end{table}

\subsection{Rewards}
\label{app:rewards}

\paragraph{Programmatic rewards.} Table~\ref{tab:prog-rewards} defines the 20 rewards.

\begin{table}[tb]
\centering
\caption{\textbf{The 20 programmatic rewards.} Each reward is $z$-scored with its mean and standard deviation over the $36{,}000$ responses of the DPO pairs. Of the 32 combined rewards, 6, 6, 8, 6 and 6 mix $k=2$, 4, 8, 10 and 16 rewards. Fifteen have hand-set weights (peaked, top weight 0.85 for $k\leq4$ and 0.82 otherwise, medium or uniform), and 17 have Dirichlet(1) weights.}
\label{tab:prog-rewards}
\small
\begin{tabular}{@{}lp{0.68\linewidth}@{}}
\toprule
Reward & Definition \\
\midrule
vocabulary diversity & type-token ratio of the words \\
repetition & share of word bigrams that are unique (higher means less repetitive) \\
stopword density & share of tokens in the NLTK English stopword list \\
lexical density & share of words outside a fixed list of function words \\
conjunction density & share of words in a list of 21 conjunctions \\
formality & mean word length, minus 10 times the contraction rate and 5 times the rate of first-person singular pronouns \\
readability & Flesch reading ease \\
mean sentence length & mean number of words per sentence \\
average word length & mean number of characters per word \\
modal density & share of words that are modal verbs \\
structure & markdown structure per line: bullets, numbered items, headers and bold text \\
paragraph structure & number of paragraphs with at least ten words \\
concreteness & numbers and capitalized words per 100 words \\
noun density & share of tokens tagged as nouns \\
code content & code markers (backticks and keywords) per 100 words \\
verb density & share of tokens tagged as verbs \\
adverb density & share of words ending in \emph{-ly} \\
personal narrative & share of words that are first-person pronouns or end in \emph{-ed} \\
polysyllabic word density & share of words with three or more syllables \\
second person & share of words that are second-person pronouns \\
\bottomrule
\end{tabular}
\end{table}

\paragraph{Reward models.}
The \rmrb RMs are \texttt{Skywork-Reward-V2-Llama-3.1-8B}~\citep{liu2026skyworkrewardv2scalingpreferencedata}, \texttt{FsfairX-LLaMA3-RM-v0.1}~\citep{dong2024rlhf}, \texttt{GRM-Llama3-8B-rewardmodel-ft}~\citep{yang2024regularizing} and \texttt{Llama-3.1-Tulu-3-8B-RM}~\citep{lambert2024tulu3}. The \rmhh experts use heads 0 (helpfulness) and 10 (safety) of \texttt{ArmoRM-Llama3-8B-v0.1}~\citep{wang2024interpretable}. The \rmpub RMs are \texttt{Skywork-Reward-V2-Llama-3.1-8B}, \texttt{internlm2-7b-reward}~\citep{cai2024internlm2}, \texttt{Eurus-RM-7b}~\citep{yuan2024advancing}, \texttt{AceMath-7B-RM}~\citep{liu2024acemath}, \texttt{OLMo-2-1124-7B-RM}~\citep{olmo2025olmo2}, \texttt{reward-model-deberta-v3-large-v2}, \texttt{oasst-rm-2.1-pythia-1.4b}~\citep{kopf2023openassistant}, \texttt{SteamSHP-flan-t5-xl}~\citep{ethayarajh2022understanding}, \texttt{Starling-RM-7B-alpha}~\citep{zhu2023starling} and \texttt{Llama8B-CreativeWritingVerifier}. Writing Sky, Fs, GRM and Tulu for the $z$-scored \rmrb scores, its eight targets are $\frac12$Sky$+\frac12$Fs, $\frac12$Sky$+\frac12$Tulu, $\frac12$Fs$+\frac12$Tulu, $0.85$Sky$+0.15$Fs, $\frac13$(Sky$+$Fs$+$Tulu), $0.6$Sky$+0.3$Fs$+0.1$Tulu, $\frac14$(Sky$+$Fs$+$Tulu$+$GRM) and $0.4$Sky$+0.3$Fs$+0.2$Tulu$+0.1$GRM. The three \rmhh targets weight helpfulness and harmlessness $0.5/0.5$, $0.75/0.25$ and $0.25/0.75$.
Table~\ref{tab:training} lists how the ten \rmpub experts were chosen.

\subsection{Training}
\label{app:training-setup}

Table~\ref{tab:training} includes the training and evaluation settings of the language bases, including how we fit the weights and coverage.

\begin{table}[tb]
\centering
\caption{\textbf{Training and evaluation settings of the language bases.} Thinking is disabled, except on tool calling. All evaluation decoding uses vLLM~\citep{kwon2023efficient}.}
\label{tab:training}
\footnotesize
\setlength{\tabcolsep}{3pt}
\begin{tabular}{@{}p{0.115\linewidth}p{0.13\linewidth}p{0.08\linewidth}p{0.09\linewidth}p{0.1\linewidth}p{\dimexpr0.36\linewidth-36pt\relax}p{0.125\linewidth}@{}}
\toprule
 & KL coefficient & lr & \raggedright batch, updates & steps & training prompts & \raggedright max.\ prompt + response \tabularnewline
\midrule
\pgrpo & 0.04, $k_3$ & $10^{-5}$ & $64\times8$, 2 & 28, 1 epoch & 1{,}518 (1{,}792 with repeats) & $512+512$ \tabularnewline
\pdpo & 0.1 & $5\times10^{-5}$ & 16 pairs & 6 epochs & 4{,}500, 4 pairs each & 512 \tabularnewline
RM bases & \raggedright $10^{-3}$, adaptive, target 0.1, $k_3$ & $3\times10^{-6}$ & $32\times8$, 4 & 232 & \raggedright 1{,}861 (\rmrb) or 1{,}979, about half PKU-SafeRLHF & $1{,}024+512$ \tabularnewline
MATH,\newline tool calling & \raggedright 0.01, as in Eq.~\eqref{eq:rlhf-obj} & $10^{-5}$ & $32\times16$, 1 & 100 & \raggedright 1{,}024 MATH or 2{,}048 RLLA-4K, as in CorrGRPO~\citep{hu2026corrgrpo} & \raggedright 1{,}024 or 2{,}048 $+\,2{,}048$ \tabularnewline
\midrule
LoRA, AdamW & \multicolumn{6}{p{\dimexpr0.885\linewidth-6pt\relax}@{}}{\raggedright rank 64, $\alpha=128$, all 7 projection matrices, dropout 0.05 on the programmatic bases, no dropout or weight decay on MATH and tool calling, gradients clipped at 1.0 in GRPO and PPO} \tabularnewline
initialization & \multicolumn{6}{p{\dimexpr0.885\linewidth-6pt\relax}@{}}{\raggedright one per basis, except one per run on \rmrb and \rmhh and three for the \pdpo targets (by $k$)} \tabularnewline
expert reward & \multicolumn{6}{p{\dimexpr0.885\linewidth-6pt\relax}@{}}{\raggedright raw on \rmrb, \rmhh, MATH and tool calling, $z$-scored elsewhere} \tabularnewline
rollouts & \multicolumn{6}{p{\dimexpr0.885\linewidth-6pt\relax}@{}}{\raggedright temperature 1.0 and top-$p$ 1.0 in GRPO, temperature 0.7, top-$p$ 0.8 and top-$k$ 20 in PPO} \tabularnewline
algorithm & \multicolumn{6}{p{\dimexpr0.885\linewidth-6pt\relax}@{}}{\raggedright \pgrpo in HybridFlow~\citep{sheng2024hybridflow}, GRPO without standard-deviation normalization on MATH and tool calling, PPO with a separate full-parameter critic (lr $10^{-5}$), GAE without discounting and 1 PPO epoch} \tabularnewline
\rmpub & \multicolumn{6}{p{\dimexpr0.885\linewidth-6pt\relax}@{}}{\raggedright experts fixed before any held-out experiment. Screen: the 600 prompts of App.~\ref{app:geometry-setup}, judged by both \texttt{ArmoRM} heads, \texttt{Qwen3Guard} and \texttt{Skywork-Reward-V2}. Dropped if another judge put it 0.5 standard deviations or more below $\pi_{\mathrm{ref}}$ on benign prompts, or if half its responses or more had one surface feature (emoji, repeated opening, refusal) at twice $\pi_{\mathrm{ref}}$'s rate or more} \tabularnewline
\midrule
\end{tabular}\par\nointerlineskip
\begin{tabular}{@{}p{0.115\linewidth}p{\dimexpr0.34\linewidth+6pt\relax}p{\dimexpr0.375\linewidth-24pt\relax}p{0.17\linewidth}@{}}
 & evaluation prompts & calibration responses & best-of-$N$ \tabularnewline
\midrule
\pgrpo,\newline \pdpo & 30, up to 96 new tokens & \raggedright 452 prompts $\times$ 4, temperature 1.0, up to 128 tokens & \raggedright 64, temperature 1.0, top-$p$ 0.95, top-$k$ 20 \tabularnewline
RM bases & \raggedright 300 (\rmrb) or 197 UltraChat and 150 PKU-SafeRLHF, up to 512 new tokens & \raggedright 4 per expert for 150 (\rmrb) or 300 training prompts, sampled as in training & \raggedright 16, sampled as in training \tabularnewline
\midrule
best-of-$N$ & \multicolumn{3}{p{\dimexpr0.885\linewidth-6pt\relax}@{}}{\raggedright $\pi_{\mathrm{ref}}$ samples per prompt, exact expectation of the best of $N$ over random subsets} \tabularnewline
top expert & \multicolumn{3}{p{\dimexpr0.885\linewidth-6pt\relax}@{}}{\raggedright largest weight in $\bm{\alpha}^*$ (combined rewards, first on ties) or $\hat{\bm{\alpha}}$ (held-out rewards)} \tabularnewline
second RL run & \multicolumn{3}{p{\dimexpr0.885\linewidth-6pt\relax}@{}}{\raggedright same recipe (\pgrpo: training prompts reordered)} \tabularnewline
GDPO & \multicolumn{3}{p{\dimexpr0.885\linewidth-6pt\relax}@{}}{\raggedright each basis reward normalized per group of rollouts, $\bm{\alpha}^*$-weighted sum whitened per batch~\citep{liu2026gdpo}} \tabularnewline
prompt pool & \multicolumn{3}{p{\dimexpr0.885\linewidth-6pt\relax}@{}}{\raggedright programmatic bases: 500 held-out UltraChat prompts for calibration and evaluation. Weight and coverage fits skip the first 50, which hold the evaluation prompts; the Gram matrix of $\gamma_{\mathrm{geom}}$ uses all calibration prompts} \tabularnewline
weight fit & \multicolumn{3}{p{\dimexpr0.885\linewidth-6pt\relax}@{}}{\raggedright combined rewards, programmatic bases: ridge with a penalty on negative weights, strength by 5-fold cross-validation over prompts, $\sum_k|\hat\alpha_k|=1$, weights below 0.01 dropped at decoding. Non-negative least squares: at most the eight largest weights, renormalized to sum to one} \tabularnewline
coverage & \multicolumn{3}{p{\dimexpr0.885\linewidth-6pt\relax}@{}}{\raggedright validation $R^2$, mean over 20 random splits holding out 20\% of prompts. Threshold 0.3, from an earlier decode of the same experts. Within-expert: per-expert indicators regressed out of the reward and the log-ratios before the ridge fit} \tabularnewline
\bottomrule
\end{tabular}
\end{table}

\subsection{Evaluation}
\label{app:eval-setup}

\paragraph{Composition strength.}
GRPO divides each reward by its standard deviation within a prompt and DPO keeps only which response of a pair wins, so both give every policy, $\pi_{\mathrm{tar}}$ included, updates of the same size whatever its reward. The weighted average of the experts' log-ratios is not rescaled in this way and shrinks when they are weakly correlated, so at $\gamma=1$ it is lower in magnitude than $\pi_{\mathrm{tar}}$'s log-ratio. PPO with the KL penalty in the reward, which trains the RM bases, keeps the scale. We define
\begin{equation}
\gamma_{\mathrm{geom}} = \Big(\frac{\sum_k|\alpha_k|\,G_{kk}}{\alpha^{\top}G\,\alpha}\Big)^{1/2},
\label{eq:gamma-geom}
\end{equation}
where $G$ is the Gram matrix of the experts' log-ratios on reference model responses, centered within each prompt, and $\alpha$ is rescaled so that $\sum_k|\alpha_k|=1$. It equals one for identical experts and $1/\|\alpha\|_2$ for uncorrelated experts of equal size.
$\gamma_{\mathrm{geom}}$ is exact when each policy optimizes Eq.~\eqref{eq:rlhf-obj} on its reward divided by the reward's standard deviation within a prompt, and the basis rewards share one standard deviation $s$, since $r_{\mathrm{tar}}$ divided by its own standard deviation then equals $\gamma_{\mathrm{geom}}\sum_k\alpha_k r_k/s$.

For the drift-matched $\gamma^\star$, let $D(\pi)$ be the drop in mean per-token log-likelihood under $\pi_{\mathrm{ref}}$ when $\pi_{\mathrm{ref}}$'s greedy responses to the evaluation prompts are replaced by those of $\pi$. With the same rescaled $\alpha$, $\gamma^\star$ solves
\begin{equation}
D\big(\pi_{\methodm(\alpha;\gamma^\star)}\big)=\sum_k|\alpha_k|\,D(\pi_k).
\label{eq:gamma-star}
\end{equation}
We decode \ours at $\gamma\in\{0.5,1,1.5,2,3\}$, set $D=0$ at $\gamma=0$, make $D$ non-decreasing in $\gamma$ with a running maximum and interpolate it linearly. If the right-hand side is above the whole curve, $\gamma^\star=3$.

\subsection{Geometry}
\label{app:geometry-setup}

We compute the effective rank on 2{,}400 reference model responses to 600 prompts, for $\widetilde R$ and $\widetilde\Phi$, standardized per column, and for the uncentered Gram matrix of the experts' weight updates $\Delta\Theta$ (flattened LoRA products $BA$), reported only for bases whose experts share one initialization. The 600 prompts are 150 UltraChat and 150 PKU-SafeRLHF prompts from the training set of \rmhh and \rmpub, and 150 of each held out from every training set. Their responses use the sampling and maximum length of the RM bases.

\FloatBarrier
\section{Additional results}
\label{app:results}

\subsection{Predicting RL outcomes for combined rewards}
\label{app:results-combined}

\FloatBarrier
\subsubsection{MATH}
\label{app:math}
\suppressfloats[t]

\begin{table}[tb]
\centering
\caption{\textbf{Rewards, targets and evaluation on MATH and tool calling.} On tool calling, RL on $r_{\mathrm{corr}}$ plus the benchmark's format check does not learn the format within 100 steps, while its own expert does, so we do not use those runs as targets.}
\label{tab:verifiable-setup}
\small
\begin{tabular}{@{}lp{0.34\linewidth}p{0.42\linewidth}@{}}
\toprule
 & MATH & Tool calling \\
\midrule
correctness reward & $r_{\mathrm{corr}}=1$ if the final boxed answer matches the reference answer, else 0 & the benchmark's $r_{\mathrm{corr}}\in[-3,3]$ scores the tool name, parameter names and values against the reference call, and is $-3$ without a valid call \\
second reward & $r_{\mathrm{len}}=\mathrm{clip}(2-n/433,0,1)$ for $n$ response tokens, where 433 is $\pi_{\mathrm{ref}}$'s median length (a graded form of GDPO's length reward~\citep{liu2026gdpo}) & $r_{\mathrm{short}}=1$ if the thinking before the tool calls has 1 to 168 words, else 0 \\
targets & $\lambda_{\mathrm{corr}}r_{\mathrm{corr}}+\lambda_{\mathrm{len}}r_{\mathrm{len}}$ with $\bm{\lambda}=(0.5,0.5)$, $(0.8,0.2)$ and $(0.2,0.8)$ & $r_{\mathrm{corr}}+w\,r_{\mathrm{short}}$ with $w\in\{0.25,1,2.263,4\}$ \\
\ours & $\bm{\alpha}^*$ at $\gamma_{\mathrm{geom}}$, from $\pi_{\mathrm{ref}}$'s responses to 256 other training problems & $\bm{\alpha}^*$ at $\gamma=1$, and $\bm{\alpha}^*/(1+w)$ at $\gamma_{\mathrm{geom}}$ from $\pi_{\mathrm{ref}}$'s responses to 256 other prompts \\
evaluation & MATH-500~\citep{lightman2023let}, 4 samples per problem & 256 held-out prompts, 8 samples per prompt \\
\bottomrule
\end{tabular}
\end{table}

We also test \ours on a verifiable reward (Table~\ref{tab:verifiable-setup}). On every combined reward, \ours recovers nearly all of $\pi_{\mathrm{tar}}$'s reward gain, in the same range as the second RL run (Fig.~\ref{fig:math}, middle). When length weighs at least as much as correctness, most of the gain comes from length, so the length expert also recovers most of it, but at a lower accuracy than $\pi_{\mathrm{tar}}$ (left). On these two mixtures \ours is also closer to $\pi_{\mathrm{tar}}$ in relative KL than the second RL run (right).

\begin{figure}[ht]
\centering
\includegraphics[width=\linewidth]{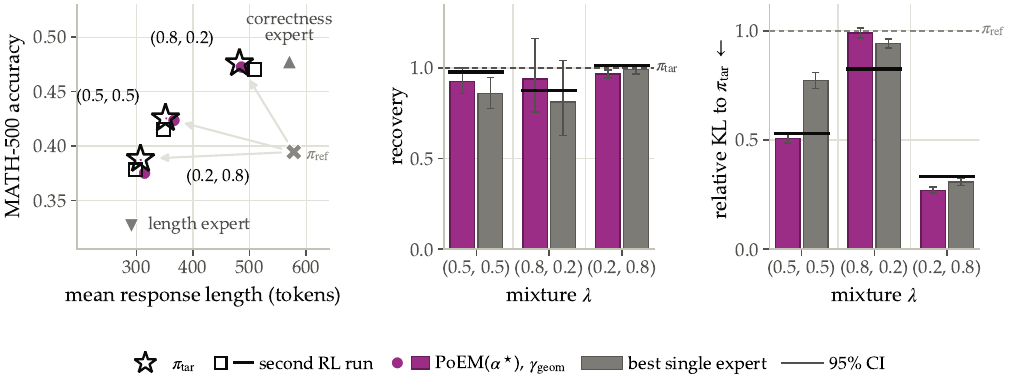}
\caption{\textbf{On MATH, \ours lands where RL does.} Left: MATH-500 accuracy against mean response length for each target $\pi_{\mathrm{tar}}$ (hollow star), its second RL run (square), \ours (dot), the two experts and $\pi_{\mathrm{ref}}$. Arrows show where RL moved $\pi_{\mathrm{ref}}$, and dotted lines join \ours to its target. Middle: recovery. Right: relative KL to $\pi_{\mathrm{tar}}$. Black lines mark the second RL run, and intervals are 95\% bootstrap intervals over the 500 problems. The best single expert is the one with the higher target reward.}
\label{fig:math}
\end{figure}

\subsubsection{Tool calling}
\label{app:toolcall}
\suppressfloats[t]

On tool calling (Table~\ref{tab:verifiable-setup}), \ours at $\gamma=1$ recovers all of $\pi_{\mathrm{tar}}$'s reward gain at the three larger weights, where short thinking saturates (Fig.~\ref{fig:behavior}), while the best single expert falls further behind as the weight grows (Fig.~\ref{fig:toolcall}, middle). At the smallest weight, the correctness expert alone recovers nearly all of the gain, but it thinks longer than $\pi_{\mathrm{tar}}$. At the two largest weights, \ours at $\gamma=1$ overshoots, with shorter thinking than $\pi_{\mathrm{tar}}$ (left) and a relative KL that grows with the weight (right). At $\gamma_{\mathrm{geom}}$ it recovers less as the weight grows, because the true weights sum to $1+w$, but it is closer to $\pi_{\mathrm{tar}}$ than either expert at every weight.

\begin{figure}[ht]
\centering
\includegraphics[width=\linewidth]{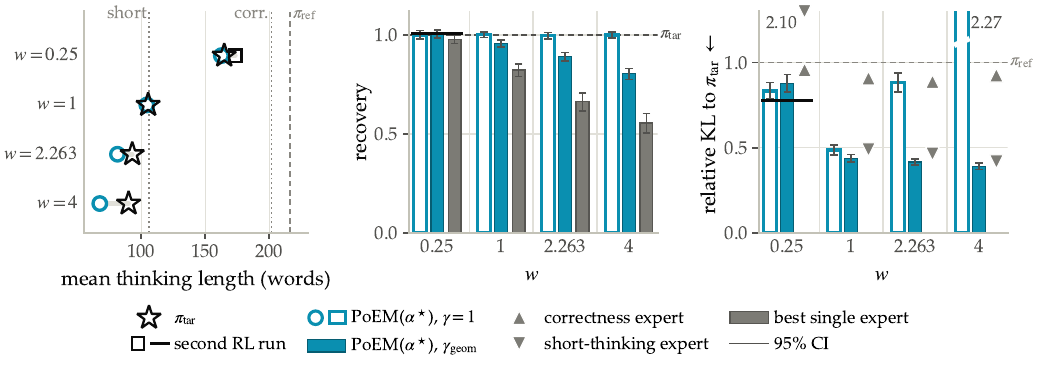}
\caption{\textbf{On tool calling, \ours matches $\pi_{\mathrm{tar}}$'s reward at $\gamma=1$ and is closer to $\pi_{\mathrm{tar}}$ than either expert at $\gamma_{\mathrm{geom}}$.} Left: mean thinking length of each target (star, and square for the second RL run at $w=0.25$) and of \ours at $\gamma=1$ (open circle). Vertical lines mark $\pi_{\mathrm{ref}}$ (dashed) and the two experts (dotted). Middle: recovery for \ours at $\gamma=1$ (open bars) and at $\gamma_{\mathrm{geom}}$ (filled), and for the best single expert (grey), which is the correctness expert except at $w=4$. Right: relative KL to $\pi_{\mathrm{tar}}$ for \ours and for both single experts (grey markers), with values above the axis printed. Black lines mark the second RL run at $w=0.25$. Intervals are 95\% bootstrap intervals over the 256 held-out prompts.}
\label{fig:toolcall}
\end{figure}

\subsubsection{Programmatic and reward model bases}
\label{app:combined-prog-rm}

Fig.~\ref{fig:combined} shows recovery and relative KL on each basis. We break recovery down by the number of composed rewards $k$ (Fig.~\ref{fig:k-dependence}) and compare \ours with decoding-time baselines (Fig.~\ref{fig:error-baselines}).

\begin{figure}[tb]
\centering
\includegraphics[width=\linewidth]{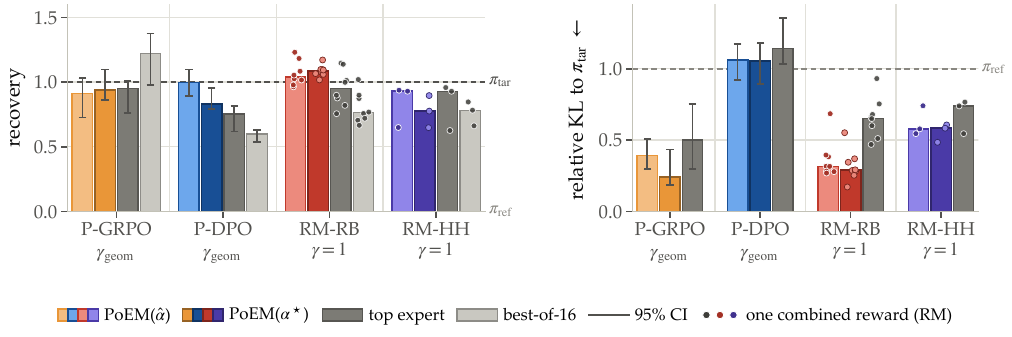}
\caption{\textbf{On combined rewards, \ours recovers most of RL's reward gain, and is closer to the target policy.} Left: $\pi_{\mathrm{tar}}$'s reward gain, where 1 matches $\pi_{\mathrm{tar}}$ and 0 the reference model. Right: KL to $\pi_{\mathrm{tar}}$ relative to that of the reference model. Bars are medians over combined rewards, with 95\% bootstrap intervals on the programmatic bases and single combined rewards as dots on the RM bases. \ours uses fitted (light) or true (dark) weights, at $\gamma_{\mathrm{geom}}$ on the programmatic bases and $\gamma=1$ on the RM bases.}
\label{fig:combined}
\end{figure}

\paragraph{Comparison with a second RL run.} We repeated RL with a different seed for combined rewards on \pdpo and \pgrpo. On the nine \pgrpo rewards, \ours$(\bm{\alpha}^*)$ at $\gamma_{\mathrm{geom}}$ has the same median reward error as the second run, but the second run is closer on six of them. Retraining $\pi_{\mathrm{tar}}$ with GDPO misses the first run by 0.09 on 6 \pgrpo combined rewards (Fig.~\ref{fig:error-baselines}).

\begin{figure}[th]
\centering
\includegraphics[width=\linewidth]{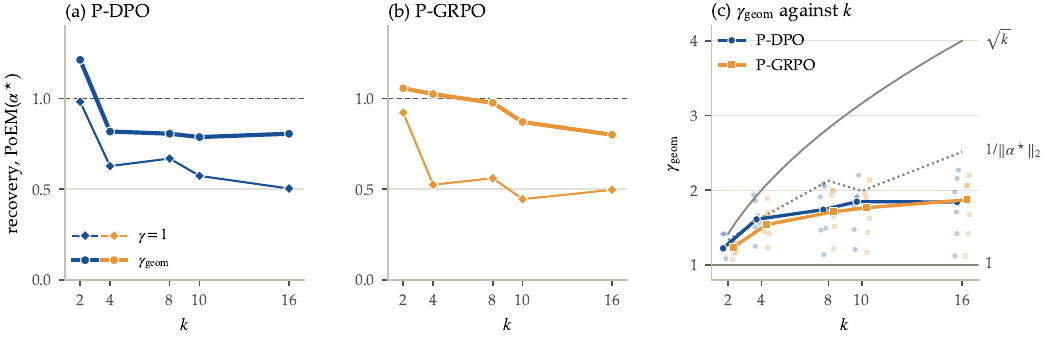}
\caption{\textbf{Recovery at $\gamma=1$ drops once $k>2$, and $\gamma_{\mathrm{geom}}$ makes up much of the drop.} (a,~b)~Median recovery of \ours$(\bm{\alpha}^*)$ at $\gamma=1$ (thin) and at $\gamma_{\mathrm{geom}}$ (thick). For $k\geq4$, $\gamma_{\mathrm{geom}}$ closes about half of the gap to $\pi_{\mathrm{tar}}$ on \pdpo and most of it on \pgrpo. (c)~$\gamma_{\mathrm{geom}}$ grows with $k$, following the value for independent experts ($\sqrt{k}$ for uniform weights) up to $k=4$ and falling below it for larger $k$, where the experts overlap.}
\label{fig:k-dependence}
\end{figure}

\begin{figure}[t]
\centering
\includegraphics[width=\linewidth]{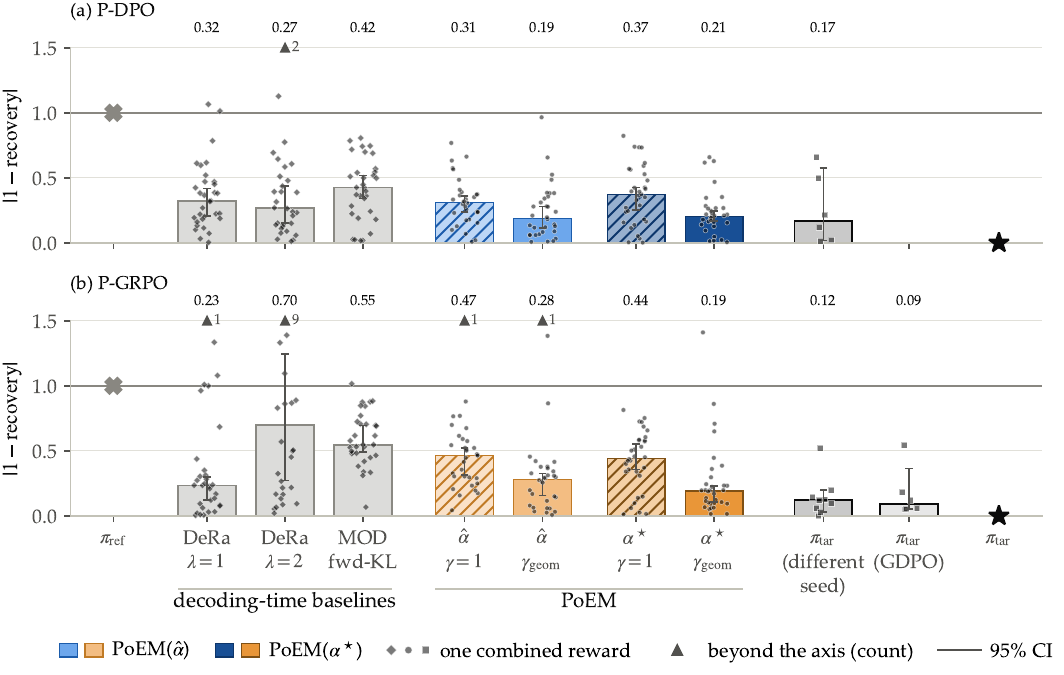}
\caption{\textbf{At $\gamma_{\mathrm{geom}}$, \ours has the lowest reward error of the decoding-time methods on \pdpo, while on \pgrpo the top expert does about as well.} DeRa~\citep{liu2024decoding} decodes the top expert alone ($\lambda=1$) or extrapolates it ($\lambda=2$). MOD~\citep{Shi2024MOD} fuses the experts under a forward-KL rule. The last columns retrain $\pi_{\mathrm{tar}}$ with another seed or, on \pgrpo, with GDPO, a multi-reward RL algorithm.}
\label{fig:error-baselines}
\end{figure}

\subsubsection{Gap between the decoder and the sequence-level product}
\label{app:decoder-gap}
\suppressfloats[t]

The decoder in Eq.~\eqref{eq:poe-decode} normalizes at every token, so its distribution $\pi_{\mathrm{tok}}$ differs from the sequence-level product $\pi_{\mathrm{seq}}$ in Eq.~\eqref{eq:poe}. The two are related exactly by $\pi_{\mathrm{seq}}(y)=\pi_{\mathrm{tok}}(y)\prod_t Z_t(h_t)/Z$, where $Z_t(h_t)$ is the decoder's normalizing constant at prefix $h_t$. Weighting decoder samples by $\prod_t Z_t(h_t)$ therefore gives an importance sample of $\pi_{\mathrm{seq}}$. We do this for three \pgrpo combined rewards at $\gamma=1$ and $\gamma_{\mathrm{geom}}$, with 64 responses per prompt at temperature 1 on 16 evaluation prompts (Table~\ref{tab:decoder-gap}). The two distributions differ more for larger $k$ and $\gamma$, yet reweighting barely moves the combined reward, and the sequence-level product is further from $\pi_{\mathrm{tar}}$ than the decoder in all six settings, significantly in five. So local normalization is not what limits \ours on this basis.

\begin{table}[H]
\centering
\caption{\textbf{The decoder against the sequence-level product.} KL is $\mathrm{KL}(\pi_{\mathrm{tok}}\|\pi_{\mathrm{seq}})$ in nats per 96-token response, and ESS is the effective sample size as a fraction of the samples. The reward change is the combined reward under $\pi_{\mathrm{seq}}$ minus that under $\pi_{\mathrm{tok}}$, in $z$-units. Relative KL is $\mathrm{KL}(\pi_{\mathrm{tar}}\|\cdot)/\mathrm{KL}(\pi_{\mathrm{tar}}\|\pi_{\mathrm{ref}})$. For each $k$, the second row uses $\gamma_{\mathrm{geom}}$.}
\label{tab:decoder-gap}
\begin{tabular}{@{}rrrrrr@{}}
\toprule
$k$ & $\gamma$ & KL & ESS & reward change & relative KL, $\pi_{\mathrm{tok}}$ / $\pi_{\mathrm{seq}}$ \\
\midrule
2 & 1 & 0.62 & 0.53 & $+0.02$ & 0.25 / 0.30 \\
2 & 1.29 & 0.65 & 0.48 & $-0.02$ & 0.35 / 0.38 \\
4 & 1 & 1.10 & 0.34 & $+0.00$ & 0.22 / 0.28 \\
4 & 1.63 & 1.48 & 0.20 & $+0.02$ & 0.19 / 0.30 \\
8 & 1 & 1.59 & 0.21 & $-0.03$ & 0.45 / 0.61 \\
8 & 2.00 & 2.72 & 0.11 & $-0.03$ & 0.51 / 0.76 \\
\bottomrule
\end{tabular}
\end{table}

\subsection{Policy space geometry}
\label{app:results-geometry}

\subsubsection{Rank on the other bases}
\label{app:rank-other-bases}
\suppressfloats[t]

\emph{Surprisingly, even when the rewards driving the experts' adapters share no obvious linear relationship, their policy log-ratios often span an approximately low-rank subspace.} Fig.~\ref{fig:rank-spectra} repeats the measurement of Fig.~\ref{fig:rank-grpo} for the ten experts of \rmpub and shows the three spaces for each basis, including \pdpo.

On the experts' own responses the \rmpub gap narrows to 6.09 against 8.00, and on \rmrb the log-ratios have lower rank than the rewards only on reference model responses.

\begin{figure}[ht]
\centering
\includegraphics[width=\linewidth]{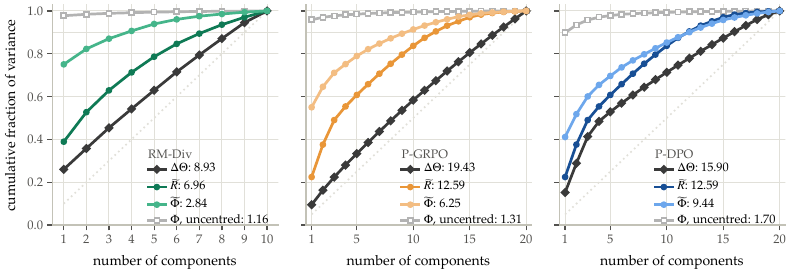}
\caption{\textbf{The gap depends on how the experts were trained.} Cumulative variance of the weight updates, rewards and log-ratios for \rmpub, \pgrpo and \pdpo, with effective ranks in the legends. Both programmatic bases show a gap.}
\label{fig:rank-spectra}
\end{figure}

\FloatBarrier
\subsubsection{Held-out experts of \texorpdfstring{\rmpub}{RM-Div}}
\label{app:rmdiv-heldout}
\suppressfloats[t]
We hold out each expert $\pi_j$ of \rmpub in turn and treat it as the policy that RL reaches for a new reward $r_j$. On responses sampled from the other nine experts for the prompts of App.~\ref{app:geometry-setup}, with every reward and policy log-ratio centered within prompt, we fit $r_j$ on their rewards and $\phi_j$ on their log-ratios by ordinary least squares, and score each fit by its $R^2$ on prompts left out of the fit. For all ten held-out experts of \rmpub, the other log-ratios explain more of the expert's log-ratio than the other nine rewards explain of its reward (median 88\% against 50\%, averaged over five 80/20 splits of the prompts). This is the gap in effective rank above, seen one expert at a time on the same pool of responses. \ours can approximate the policy that RL reaches for a new reward when that policy's log-ratio lies near the span of the basis log-ratios, whether or not the reward lies in the span of the basis rewards. The held-out experts meet the first condition more closely than their rewards meet the second. At $\gamma=1$ on the evaluation prompts, \ours$(\hat{\bm{\alpha}})$ recovers 0.43 and 0.57 of the gains of the two experts in Fig.~\ref{fig:rmdiv-heldout-fits}.

\begin{figure}[bt]
\centering
\includegraphics[width=\linewidth]{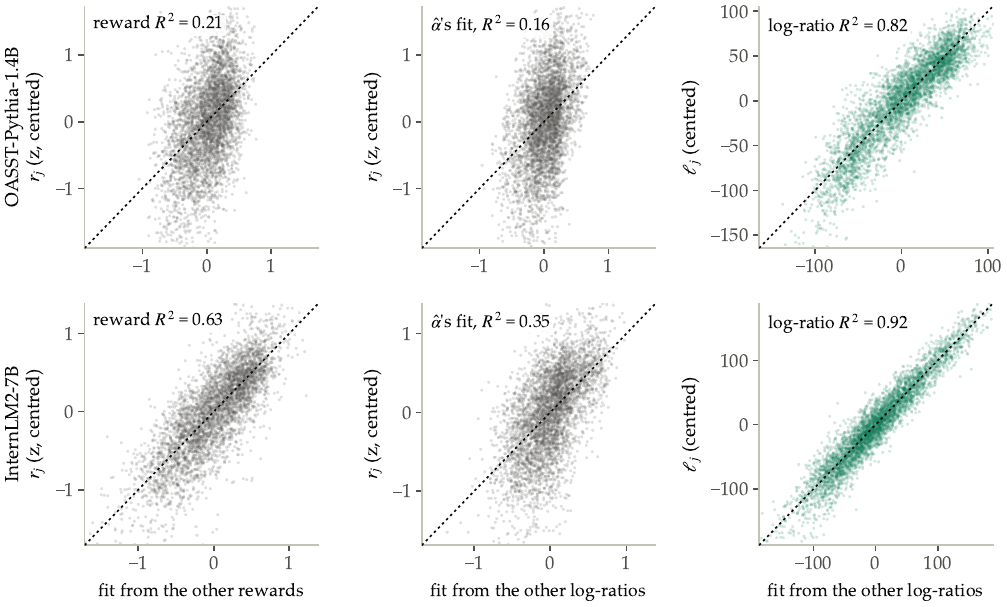}
\caption{\textbf{A held-out expert's log-ratio can be close to a combination of the others' when its reward is not.} Two held-out models of \rmpub, \texttt{OASST-Pythia-1.4B} (top) and \texttt{InternLM2-7B} (bottom), on held-out prompts, with every quantity centered per prompt. Left: the target reward $r_j$ against its fit from the other nine rewards. Middle: the same reward against a least-squares fit from the other nine policy log-ratios. Right: the held-out expert's log-ratio $\phi_j=\log\pi_j-\log\pi_{\mathrm{ref}}$ against its fit from the other log-ratios.}
\label{fig:rmdiv-heldout-fits}
\end{figure}

\subsubsection{Size and construction of the calibration set}
\label{app:calibration}
\suppressfloats[t]

On the programmatic bases, $\hat{\bm{\alpha}}$ is fit on 323 of the 404 calibration prompts, with 4 reference model responses each. We refit it on smaller or differently built calibration sets and compare the weights with those of the full set (Fig.~\ref{fig:calibration}).\footnote{The gradient solver we use for $\hat{\bm{\alpha}}$ does not converge on small \pgrpo subsets, so this ablation minimizes the same objective exactly. On the full set, the exact minimizer has cosine 0.988 or higher to the solver's weights.} With responses centered within each prompt, about 100 prompts suffice. Their weights have cosine 0.97 on \pdpo and 0.94 on \pgrpo to those of the full set, against 0.85 and 0.74 with 25 prompts. Centering within prompt matters most on \pgrpo. With 320 responses and no centering, 320 prompts with one response each give cosine 0.80 and 0.38. Which prompts are used matters less. Fitting on one of four groups of prompts, split by topic, prompt length, $\pi_{\mathrm{ref}}$ likelihood, response style or the leading direction of the policy log-ratios, instead of on a random set of the same size, changes the cosine by at most 0.03 on average and 0.07 for any group. We also decode \ours$(\hat{\bm{\alpha}})$ at $\gamma_{\mathrm{geom}}$ with weights from three random draws of 25 and of 100 prompts. Compared with the full set, recovery drops by 0.18 on \pdpo and 0.26 on \pgrpo with 25 prompts and by 0.10 and 0.03 with 100. At $\gamma=1$, weights from 100 prompts match the full set ($+0.01$ and $-0.02$), so on \pdpo their drop comes from a smaller $\gamma_{\mathrm{geom}}$ (1.52 against 1.64).

\begin{figure}[H]
\centering
\includegraphics[width=0.87\linewidth]{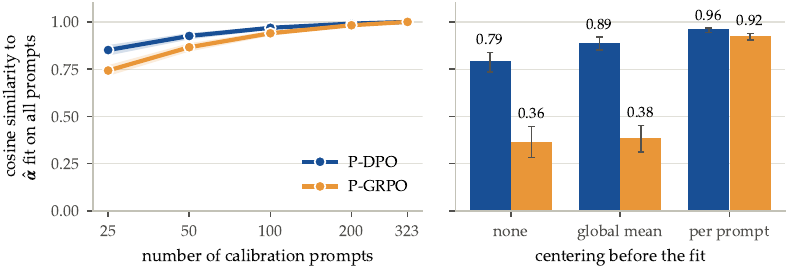}
\caption{\textbf{About 100 calibration prompts suffice when responses are centered within each prompt.} Cosine between $\hat{\bm{\alpha}}$ refit on part of the calibration set and $\hat{\bm{\alpha}}$ fit on all 323 prompts, mean over 32 combined rewards and 50 draws, with 95\% bootstrap intervals over combined rewards. Left: number of prompts, with 4 responses each, centered within prompt. Right: 80 prompts with 4 responses each, not centered, centered by one global mean, or centered within each prompt.}
\label{fig:calibration}
\end{figure}

\FloatBarrier
\subsection{Predicting RL outcomes for held-out rewards}
\label{app:results-heldout}

\subsubsection{A large reward model}
\label{app:large-rm}

\paragraph{Setup.} The target reward is \texttt{Skywork-Reward-V2-Llama-3.1-8B}, which scores 84.1 on RewardBench~2~\citep{malik2026rewardbench}, and $\pi_{\mathrm{tar}}$ is the expert trained on it (the \texttt{Skywork} expert of \rmpub). \rmsky is a new basis of five experts, one for each smaller RM of the same \texttt{Skywork-Reward-V2} family~\citep{liu2026skyworkrewardv2scalingpreferencedata}, built on \texttt{Qwen3-0.6B}, \texttt{Llama-3.2-1B}, \texttt{Qwen3-1.7B}, \texttt{Llama-3.2-3B} and \texttt{Qwen3-4B}. These RMs were trained on the same preference data as the 8B RM and score 61 to 76 on RewardBench~2. The second basis is \rmpub without its \texttt{Skywork} expert: nine experts trained on RMs from other groups, with 0.4B to 8B parameters and, for the four listed on RewardBench~2, scores of 26 to 58.

\paragraph{Results.} Over \rmsky, \ours recovers all of the reward gain and is closer to $\pi_{\mathrm{tar}}$ than a second RL run on the 8B RM on 325 of 347 prompts. $\hat{\bm{\alpha}}$ puts most of its weight on the \texttt{Llama-3.2-3B} expert and the rest on the \texttt{Qwen3-4B} expert. Both experts get the reward on their own, but neither is as close to $\pi_{\mathrm{tar}}$ as their composition. Over \rmpub, \ours recovers 0.60 at relative KL 0.83, although every one of its nine experts alone is farther from $\pi_{\mathrm{tar}}$ than $\pi_{\mathrm{ref}}$ is, and best-of-16 recovers 0.70. We can see why the two bases differ by looking at the rewards themselves. A non-negative combination of the five \rmsky RMs explains more of the within-prompt variance of the 8B RM's reward than the nine \rmpub RMs (held-out $R^2$ of 0.87 against 0.55 on reference model responses, and 0.84 against 0.43 for the best single RM of each basis). Pooling all 14 experts recovers 0.88, below its top expert alone (\texttt{Llama-3.2-3B}). These results come from one target RM.

\FloatBarrier
\subsubsection{Coverage and composition strength}
\label{app:coverage-strength}
\suppressfloats[t]

We examine whether increasing composition strength can compensate for low coverage (Table~\ref{tab:heldout-summary}, Figs.~\ref{fig:heldout-coverage} and~\ref{fig:heldout-gamma}). On \pgrpo, as $\gamma$ grows from 0.5 to 2, the median covered prediction's recovery rises from 0.18 to 0.64, while the median uncovered one reaches only 0.26. With $\gamma$ tuned on $\pi_{\mathrm{tar}}$ for each reward, uncovered rewards still recover much less than covered ones on both programmatic bases (Fig.~\ref{fig:heldout-gamma}).

\begin{table}[ht]
\centering
\caption{\textbf{Coverage ranks the difficulty of predicting RL on held-out rewards.} On \rmpub, covered means within-expert coverage above its median (Table~\ref{tab:training}). \ours uses $\gamma=1$ on \rmpub and $\gamma_{\mathrm{geom}}$ elsewhere. Relative KL and closer are as in Table~\ref{tab:combined-summary}.}
\label{tab:heldout-summary}
\small
\setlength{\tabcolsep}{4.5pt}
\begin{tabular}{@{}lcccccc@{}}
\toprule
 & \multicolumn{3}{c}{Covered} & \multicolumn{3}{c}{Not covered} \\
\cmidrule(lr){2-4}\cmidrule(l){5-7}
Basis & recovery & relative KL & closer & recovery & relative KL & closer \\
\midrule
\pgrpo & 0.55 & 0.47 & 9/9 & 0.22 & 1.04 & 5/11 \\
\pdpo & 0.69 & 1.28 & 3/9 & 0.29 & 2.01 & 0/11 \\
\rmpub & 0.60 & 0.64 & 4/5 & 0.43 & 0.99 & 3/5 \\
\bottomrule
\end{tabular}
\end{table}

\begin{figure}[ht]
\centering
\includegraphics[width=\linewidth]{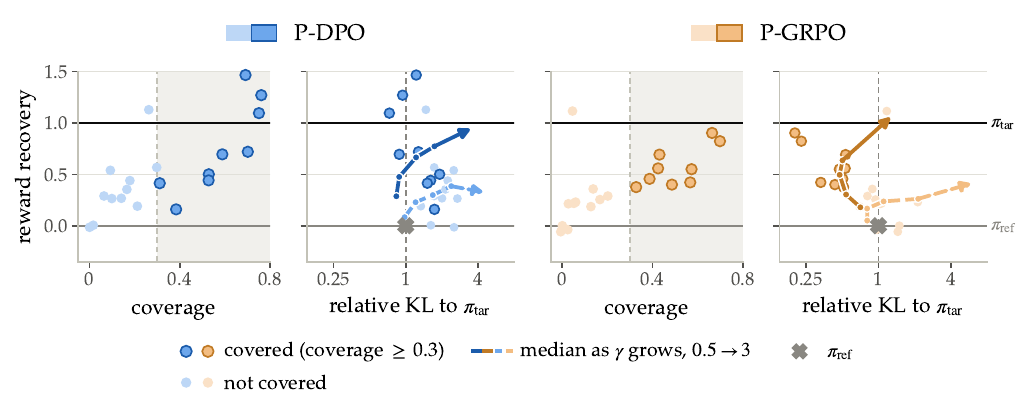}
\caption{\textbf{Coverage predicts which held-out rewards \ours can reach.} Each dot is one expert, held out and predicted by composing the other 19 at $\gamma_{\mathrm{geom}}$. Arrows show how the median prediction moves as $\gamma$ grows.}
\label{fig:heldout-coverage}
\end{figure}

\begin{figure}[tb]
\centering
\includegraphics[width=\linewidth]{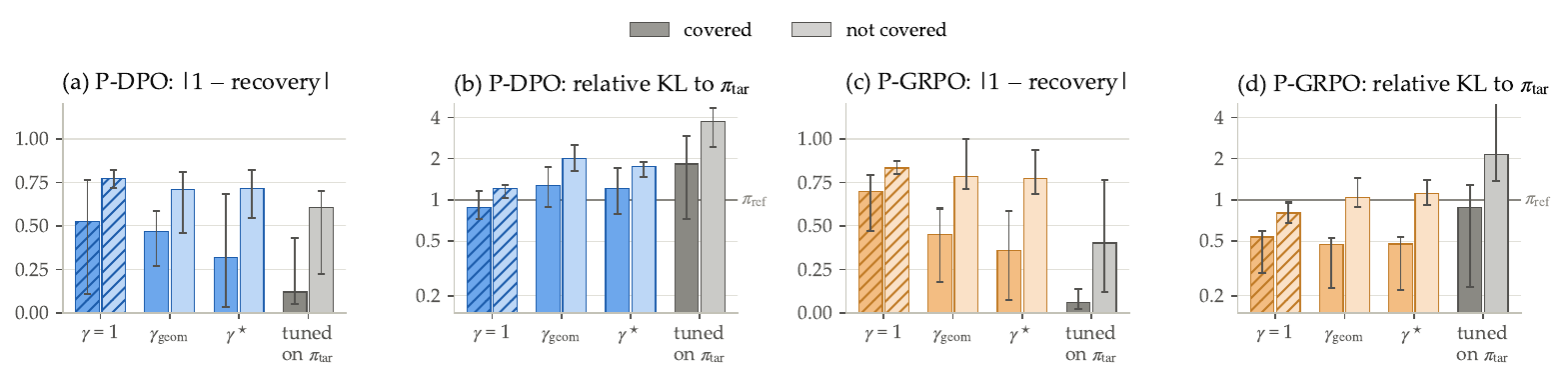}
\caption{\textbf{Uncovered rewards miss a direction, not $\gamma$.} Median reward error $|1-\text{recovery}|$ (a,~c) and relative KL to $\pi_{\mathrm{tar}}$ (b,~d) for covered (dark) and uncovered (light) held-out rewards. The grey bars tune $\gamma$ on $\pi_{\mathrm{tar}}$ for each reward. On \pgrpo, every covered reward ends closer to $\pi_{\mathrm{tar}}$ than $\pi_{\mathrm{ref}}$ at $\gamma=1$, $\gamma_{\mathrm{geom}}$ and $\gamma^\star$.}
\label{fig:heldout-gamma}
\end{figure}

\subsubsection{Reward model bases}
\label{app:heldout-rm}
\suppressfloats[t]

On \rmpub, \ours is closer to the held-out expert than the top expert on 9 of 10 rewards and recovers more of its reward on 8 (median 0.56 against 0.44). Best-of-16 reaches 0.60, and uniform weights 0.50.

Within-expert coverage, which we chose before decoding, ranks recovery across the ten held-out \rmpub experts better than the per-prompt coverage of Eq.~\eqref{eq:coverage-score} (Spearman $\rho=0.81$ against 0.45), and it ranks relative KL better too. On \rmrb both scores are below 0.12 for all four held-out experts, yet \ours recovers 0.56 to 0.83 of their gain.

All within-expert scores are below the 0.3 threshold of Fig.~\ref{fig:heldout-coverage}, so we split \rmpub at its median. On the more-covered half, $\gamma^\star$ increases median recovery from 0.60 to 0.66 and relative KL from 0.64 to 0.72. On the less-covered half, recovery goes from 0.43 to 0.51 and relative KL from 0.99 to 1.28, and tuning $\gamma$ on $\pi_{\mathrm{tar}}$ reaches 0.67 at a relative KL of 2.04. On \rmrb, $\gamma^\star$ barely changes median recovery.

\subsection{Image diffusion models}
\label{app:diffusion}
\suppressfloats[t]

The image RMs are aesthetic score~\citep{murray2012ava}, \texttt{PickScore}~\citep{pickscore} and CLIP score~\citep{radford2021learning}. The diffusion experts follow the compressibility configuration of the public DDPO implementation~\citep{black2023training} on \texttt{Stable Diffusion v1.4}~\citep{rombach2022high}, and train for 42 epochs instead of 100. This is the configuration of the RMS-contrast expert, and the other experts use the same code. On \sdddpo, we sample with DDIM for 50 steps at $\eta=0$ and classifier-free guidance 5.0, from ``a photo of a red fox'', ``a photo of a bald eagle'' and ``a photo of a lion'', one seed each. The composition weights are a least-squares fit, with an intercept.

Table~\ref{tab:diffusion-heldout} holds out each of the 13 image experts in turn. Here the weights are fit to the held-out expert itself, so recovery measures how much of it the other experts span, not how well they predict it. Span $R^2$ is near 1 for every expert because it includes the part of the noise prediction shared with the reference model. Fig.~\ref{fig:rewards-matrix-diffusion} shows how image rewards vary across images.

\begin{table}[ht]
\centering
\caption{\textbf{Composing the other twelve experts recovers more than half of the held-out expert's reward gain on 9 of the 13 image rewards.} Each row holds out one expert. Recovery is averaged over three prompts, and span $R^2$ is the share of the held-out expert's noise prediction that the per-step fit explains, averaged over denoising steps and prompts. $^\dagger$On one prompt the expert barely improves over the reference model, so recovery there is unstable (we keep it in the mean). $^\ddagger$The expert lowers CLIP score on two of the three prompts, and recovery then measures how well \ours reproduces the decrease.}
\label{tab:diffusion-heldout}
\begin{tabular}{@{}lrr@{\hspace{2.5em}}lrr@{}}
\toprule
held-out reward & recovery & span $R^2$ & held-out reward & recovery & span $R^2$ \\
\midrule
CLIP score$^\ddagger$ & 1.07 & 0.995 & BRISQUE & 0.96 & 0.994 \\
edge density & 0.86 & 0.991 & entropy & 0.85 & 0.995 \\
incompressibility & 0.79 & 0.995 & compressibility & 0.76 & 0.995 \\
saturation & 0.74 & 0.996 & symmetry & 0.72 & 0.994 \\
RMS contrast & 0.54 & 0.993 & colorfulness & 0.47 & 0.995 \\
\texttt{PickScore}$^\dagger$ & 0.33 & 0.994 & sharpness & 0.33 & 0.994 \\
aesthetic score$^\dagger$ & 0.25 & 0.995 & mean & 0.67 & 0.994 \\
\bottomrule
\end{tabular}
\end{table}

\begin{figure}[ht]
    \centering
    \includegraphics[width=\linewidth]{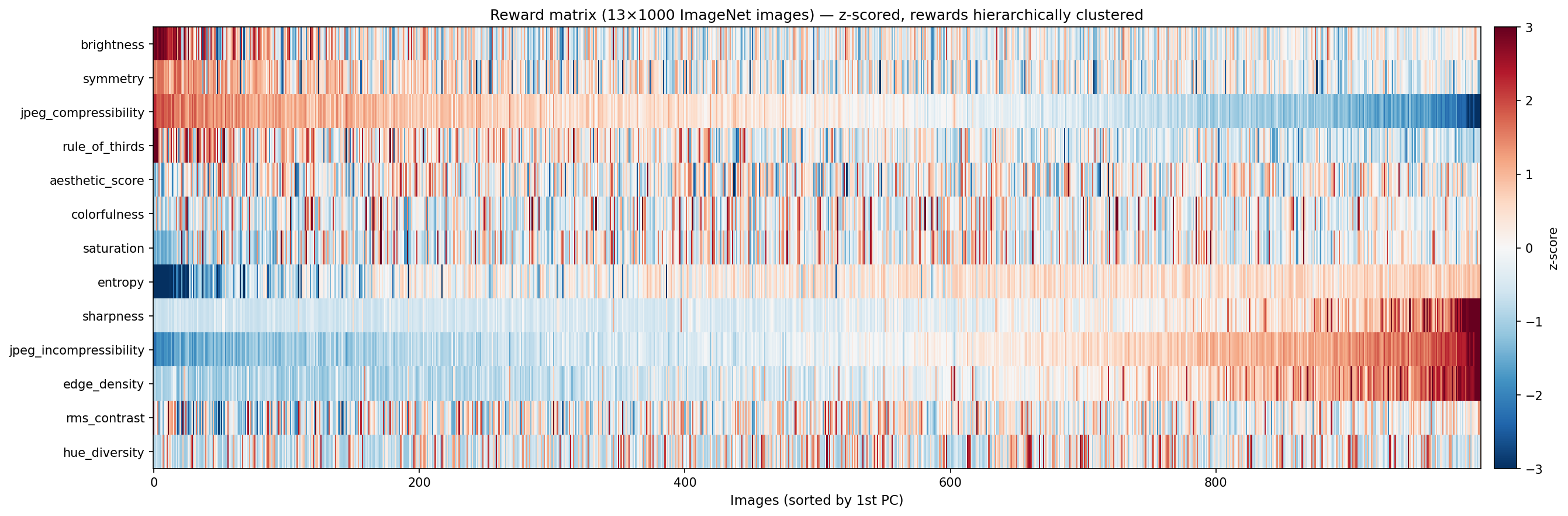}
    \caption{\textbf{Most diffusion rewards vary independently of one another.} Each row is one $z$-scored image reward on 1{,}000 ImageNet photographs, sorted by the first principal component. Brightness, hue diversity and rule of thirds have no expert in \sdddpo, and BRISQUE, \texttt{PickScore} and CLIP score are not shown. Only one group tied to image detail (sharpness, edge density, entropy and the two JPEG rewards) moves together.}
    \label{fig:rewards-matrix-diffusion}
\end{figure}

%%%%%%%%%%%%%%%%%%%%%%%%%%%%%%%%%%%%%%%%%%%%%%%%%%%%%%%%%%%%

\clearpage

\end{document}